\documentclass{article}

\usepackage[numbers,sort&compress]{natbib}

\usepackage[margin=1in]{geometry}
\usepackage{amsmath,amsfonts,amssymb,amsthm}
\usepackage{mathtools}
\allowdisplaybreaks[2]
\usepackage{mathrsfs}
\usepackage{dsfont}
\usepackage{graphicx}
\usepackage{booktabs}
\usepackage{tabularx}
\usepackage{array}
\usepackage{enumitem}
\usepackage{algorithm}
\usepackage{algpseudocode}
\usepackage{xcolor}
\usepackage{tikz}
\usepackage[expansion=false]{microtype}
\usepackage{multirow}
\usepackage[colorlinks,citecolor=blue,urlcolor=blue,linkcolor=blue]{hyperref}

\newcommand{\safeinput}[1]{%
  \IfFileExists{#1}{\input{#1}}{%
    \makebox[\textwidth][c]{\textit{Missing external table file}}%
  }%
}
\newcommand{\safeincludegraphics}[2][]{%
  \IfFileExists{#2}{\includegraphics[#1]{#2}}{%
    \fbox{\parbox[c][0.22\textheight][c]{0.8\linewidth}{\centering Missing figure file:\\ \texttt{\detokenize{#2}}}}%
  }%
}

\usetikzlibrary{arrows.meta,positioning,fit,calc,backgrounds,shapes.geometric}

\definecolor{rbblue}{RGB}{52,104,176}
\definecolor{rbgreen}{RGB}{55,145,110}
\definecolor{rborange}{RGB}{206,128,44}
\definecolor{rbpurple}{RGB}{118,88,164}
\definecolor{rbred}{RGB}{185,75,75}
\definecolor{rbgray}{RGB}{80,80,80}

\pgfdeclarelayer{background}
\pgfsetlayers{background,main}

\newcolumntype{L}{>{\raggedright\arraybackslash}X}
\renewcommand{\arraystretch}{1.4}

\theoremstyle{plain}
\newtheorem{theorem}{Theorem}

\newtheorem{lemma}{Lemma}

\theoremstyle{definition}

\theoremstyle{definition}
\newtheorem{assumption}{Assumption}

\theoremstyle{remark}
\newtheorem{remark}{Remark}

\newcommand{\bx}{\boldsymbol{x}}
\newcommand{\by}{\boldsymbol{y}}

\newcommand{\bR}{\mathbb{R}}
\newcommand{\NN}{\mathrm{NN}}
\newcommand{\cX}{\mathcal{X}}

\newcommand{\barcX}{\overline{\mathcal{X}}}
\newcommand{\barbeta}{\bar{\beta}}

\newcommand{\R}{\mathbb{R}}

\newcommand{\cF}{\mathcal{F}}

\begin{document}

\title{A Reverse-BSDE Diffusion Sampler}
\author{%
Jairon H. N. Batista$^{1}$,\quad Fl\'avio B. Gon\c{c}alves$^{2}$\\
Yuri F. Saporito$^{1}$,\quad Rodrigo S. Targino$^{1}$\\[0.6em]
\small $^{1}$School of Applied Mathematics, FGV EMAp, Funda\c{c}\~ao Getulio Vargas\\
\small $^{2}$Department of Statistics, Institute of Exact Sciences, Universidade Federal de Minas Gerais\\[0.4em]
\small \href{mailto:jairon.batista@fgv.edu.br}{jairon.batista@fgv.edu.br};
\href{mailto:yuri.saporito@fgv.br}{yuri.saporito@fgv.br};
\href{mailto:rodrigo.targino@fgv.br}{rodrigo.targino@fgv.br}\\
\small \href{mailto:fbgoncalves@ufmg.br}{fbgoncalves@ufmg.br}
}
\date{}
\maketitle

\begin{abstract}
Diffusion-based generative models have renewed interest in stochastic differential equation methods for sampling from complex distributions. We study a setting in which the target density is known only up to a normalizing constant and reformulate the reverse-time diffusion sampler as a forward-backward stochastic differential equation (FBSDE). This formulation replaces the separate pre-estimation of the time-dependent score with the solution of a coupled stochastic system. We prove the equivalence of these formulations and provide a decomposition of the approximation error arising from initializing the sampler with a standard Gaussian, applying Euler discretization to the dynamics, and solving the FBSDE approximately. We then evaluate the proposed algorithm on synthetic targets, including separated mixtures, anisotropic Gaussian distributions, and banana-shaped and ring-shaped distributions. The results demonstrate the promise of the method, particularly for targets with complex global structure.
\end{abstract}


\section{Introduction}

Generative models have emerged as powerful tools for learning complex data distributions, with applications spanning from image synthesis and natural language processing to scientific modeling (\cite{song2021}, \cite{brown2020languagemodelsfewshotlearners}). Among the diverse approaches, diffusion-based generative models stand out due to their ability to generate structured data from simple noise distributions. These models rely on stochastic processes to transform simple reference distributions into data-like distributions, effectively capturing multimodal structures that are challenging to represent with traditional methods.

A foundational work in this domain is \cite{song2021}, which established a framework for training diffusion models by solving stochastic differential equations (SDEs) to approximate the underlying data distribution. This approach reframes the problem as one of learning a ``score function'' (i.e. the gradient of the data distribution's log density) using the reverse SDE studied in \cite{ANDERSON1982}. Although diffusion-based generative models have been highly successful, most research inspired by this work has focused on improving the learning process, leaving the potential of these models for tackling sampling problems underexplored.

In this work, we take a different perspective by shifting the focus from ``learning'' the data distribution to ``sampling'' from it. Specifically, we reinterpret the generative process as a forward-backward stochastic differential equation (FBSDE); for an overview of the area, see for instance \cite{BSDEbook}. This reformulation allows us to directly address the problem of sampling from distributions whose densities are only known up to a normalizing constant, a scenario frequently encountered in Bayesian statistics and other probabilistic domains.

To make this approach practical, we propose solving the FBSDE numerically using deep learning \cite{DeepBSDE}. By parametrizing the forward and backward components of the equation with neural networks, we develop a method that approximates the sampling process for multimodal and geometrically nontrivial target distributions. This approach provides a FBSDE-based route to training diffusion samplers directly from an unnormalized target density, offering a complementary perspective on the use of diffusion processes for statistical sampling.


There are a few works exploring the use of this framework for sampling, as, for instance, \cite{ReverDiffusionMonteCarlo} and \cite{DiffusionPINN}. In both studies, the reverse SDE is employed to sample from a distribution known up to a normalizing constant. The main challenge lies in estimating the score function. In \cite{ReverDiffusionMonteCarlo}, the authors reformulate the score function as an expectation with respect to an appropriate distribution and then use Monte Carlo methods (such as Importance Sampling and Langevin Dynamics) to approximate this expectation. Meanwhile, in \cite{DiffusionPINN}, the author estimates the score function by solving the log-Fokker--Planck equation that the score satisfies, leveraging numerical methods for PDEs based on deep learning \cite{PINN}.

A closely related line of work is based on controlled diffusions and Schrödinger--Föllmer flows. In particular, \cite{vargas2023bayesian} formulate Bayesian inference through a finite-horizon stochastic-control problem, parametrizing the Föllmer drift with a neural network in order to construct a controlled diffusion whose terminal distribution approximates the posterior. This perspective is related to ours in that both approaches use finite-horizon stochastic dynamics with a learned drift or control, rather than relying on a target distribution as the invariant law of a long-run Markov chain. However, our construction is different: we start from the reverse-time diffusion associated with the noising process and identify the unknown time-dependent score with the $Z$ process of a coupled FBSDE. Thus, rather than optimizing a Schrödinger--Föllmer control objective, we train the reverse sampler through FBSDE residuals tied to the terminal unnormalized log-density.

Our contributions can be summarized as providing a novel perspective on diffusion-based generative models by re-framing the generative process as an FBSDE for the sampling problem. We also distinguish between the exact continuous-time formulation and the practical finite-horizon sampler. While the exact reverse process is initialized from the finite-time noised distribution, the implementable sampler starts from a standard Gaussian approximation. We then derive an error decomposition that separates the effects of this initialization bias, the Euler discretization of the dynamics, and the approximate numerical solution of the FBSDE. Finally, we discuss training diagnostics inspired by recent deep BSDE methods. The numerical experiments evaluate the method on a range of synthetic targets, showing its ability to capture multimodal, anisotropic, banana-shaped, and ring-shaped distributions.

\section{Reverse diffusions}

In this section we review the methodology of \cite{song2021}, which is based on the well-known result of \cite{ANDERSON1982}, when applied to the problem of sampling from $p_0(\bx) = \pi(\bx)/c$, where $c$ is an unknown constant, but $\pi$ is computable at any $\bx \in \bR^n$. 

Under suitable regularity conditions on $p_0$, which will be discussed below, the authors of \cite{Haussmann1986} stated that if $\beta:[0,T] \to \bR_{>0}$ is a scalar, continuous function and $(\mathcal{X}_t)_{t \geq 0}$ is the following diffusion process,
\begin{align} \label{eq:forward}
\begin{cases}
d\cX_t = -\tfrac{1}{2}\beta^2(t)\cX_t dt + \beta(t)dW_t, \\
\cX_0 \sim p_0,
\end{cases}
\end{align}
where $(W_t)_{t \geq 0}$ is a $n$-dimensional Brownian motion independent of $\cX_0$, then the reverse process $\barcX_t = \cX_{T-t}$ is also a diffusion process, and satisfies
\begin{align}\label{eq:reverse}
d \barcX_t
={}& \Bigl(\tfrac{1}{2}\barbeta^2(t)\barcX_t
      + \barbeta^2(t)\nabla\log p(T-t,\barcX_t)\Bigr)dt + \barbeta(t)d\overline{W}_t,
\end{align}
where $\barbeta(t) = \beta(T-t)$, $p(t,\cdot)$ is the density of $\cX_t$, $(\overline{W}_t)_{t \geq 0}$ is another $n$-dimensional Brownian motion, and $\barcX_0 \sim p(T,\cdot)$ is independent of $(\overline{W}_t)_{t \geq 0}$. More complex dynamics of $\cX$ could be considered, as stated in Appendix~\ref{app:reverse}.
 
Based on this result, the idea of the sampling algorithm in \cite{song2021} is that, if we sample $\barcX_0$ from a standard Gaussian, then, for large $T$, $\barcX_T$ will be approximately distributed as $p_0$, since the limiting distribution, when $t \to +\infty$, of the process $(\cX_t)_{t \geq 0}$ is standard Gaussian. The main challenge in applying this procedure is to estimate the score function $\nabla \log p$ that appears in the drift of \eqref{eq:reverse}. In \cite{song2021}, under a different setting where i.i.d.\ samples from $p_0$ are available, the authors reformulate the $L^2$-loss for score estimation so that only samples from $p_0$ are required.

For our setting, when $p_0$ is computable up to a normalizing constant, one needs to consider a different approach. Note that, under regularity conditions, $p$ is the unique smooth solution of the Fokker--Planck PDE:
\begin{align}\label{eq:FP}
\partial_t p(t,\bx)
- \frac{1}{2}\beta^2(t)
\Bigl[&\bx\cdot\nabla p(t,\bx)
      + n\,p(t,\bx) + \Delta p(t,\bx)\Bigr]
= 0,
\end{align}
with $p(0,\bx) = p_0(\bx)$. 

We are then ready to state the following standard assumptions, which will be in force throughout the paper:

\begin{assumption}[Integrability of \(p_0\)]
\label{ass:integrability}
There exists \(\lambda>0\) such that
\[
\int_{\mathbb R^n}
\frac{p_0(x)^2}{(1+\lVert x\rVert^2)^\lambda}
\,\mathrm dx
<\infty.
\]
\end{assumption}

\begin{assumption}[Smoothness and positivity]
\label{ass:smoothness}
The Fokker--Planck equation~\eqref{eq:FP}, with initial
condition \(p_0\), admits a unique solution
\[
p\in C^{1,2}\bigl([0,T]\times\mathbb R^n\bigr),
\]
and \(p(t,x)>0\) for every
\((t,x)\in[0,T]\times\mathbb R^n\).
\end{assumption}

\begin{assumption}[Lipschitz continuity of the score]
\label{ass:lipschitz-control}
There exists \(L_s<\infty\) such that
\[
\lVert \bar\beta(t)\nabla\log p(T-t,x)-\bar\beta(t)\nabla\log p(T-t,y)\rVert
\le L_s\lVert x-y\rVert
\]
for every \(t\in[0,T]\) and \(x,y\in\mathbb R^n\).
\end{assumption}

\begin{remark}
Assumption~\ref{ass:integrability} holds if $p_0$ is bounded, while Assumption~\ref{ass:smoothness} holds if $p_0$ is bounded and $C^2$. The validity of Assumptions~\ref{ass:integrability} and~\ref{ass:smoothness} for $p_0$ is equivalent to their validity for $\pi$.
\end{remark}

Assumption~\ref{ass:integrability} together with the continuity of $\beta : [0,T] \rightarrow \R_{>0}$ guarantees the well-posedness of Equation~\eqref{eq:reverse}, which is discussed in Appendix~\ref{app:reverse}.

One could pursue a numerical solution of this Fokker--Planck PDE to estimate the score $\nabla \log p$. However, as the normalizing constant $c$ of the density $p_0$ is unknown, it is not directly possible to compute the initial condition associated with the Fokker--Planck PDE in \eqref{eq:FP}. In order to overcome this issue, we consider a reparametrization of the problem and the PDE associated with this transformation of $p$. Namely, we study the equation satisfied by
\begin{equation}\label{eq:u}
   u(t,\bx) = \log p(T-t, \bx) - \frac{n}{2} \int_{0}^{T-t} \beta^2(s) ds + \log c,
\end{equation}
which is given by \textit{semi-linear} PDE:
\begin{align}\label{eq:uPDE}
&\partial_t u(t,\bx) + \barbeta^2(t)
\left[\frac{1}{2}\bx+\nabla u(t,\bx)\right]
\cdot\nabla u(t,\bx) + \frac{1}{2}\barbeta^2(t)\Delta u(t,\bx)
- \frac{1}{2}\|\barbeta(t)\nabla u(t,\bx)\|^2
=0,
\end{align}
with final condition $u(T,\bx) = \log \pi(\bx)$. The first important characteristic of the PDE satisfied by $u$ is that its boundary condition depends only on known functions (no dependence on $c$). Second, the way the terms are arranged matches the drift and the volatility of the reverse diffusion.

\subsection{Other approaches}

In \cite{DiffusionPINN} the authors explore numerical solutions to the non-linear PDE in \eqref{eq:uPDE}. With this approximation for $u$, one can compute its gradient (the score) and simulate the reverse diffusion starting from a standard Gaussian distribution, generating approximate samples from $p_0$.

Another way to represent the score, when, for simplicity, $\beta^2 \equiv 2$, was given in \cite{ReverDiffusionMonteCarlo}:
\begin{align*}
\nabla\log p(t,\bx)
={}&
\frac{
\mathbb E_{\bx_0\sim\rho_t(\cdot\mid\bx)}
\left[
\dfrac{e^{-t}\bx_0-\bx}{1-e^{-2t}}\,\pi(\bx_0)
\right]
}{
\mathbb E_{\bx_0\sim\rho_t(\cdot\mid\bx)}
\left[\pi(\bx_0)\right]
},
\end{align*}
where
\begin{align*}
\rho_t(\bx_0\mid\bx)
&\propto
\exp\left(
-\frac{\|\bx-e^{-t}\bx_0\|^2}
       {2(1-e^{-2t})}
\right).
\end{align*}
The authors of \cite{ReverDiffusionMonteCarlo} then use either unadjusted Langevin Dynamics or self-normalized importance sampling to estimate the score.

\section{Reverse-BSDE diffusion sampler}

Unlike the aforementioned references, we propose here a novel approach to explore the PDE in \eqref{eq:uPDE} by recasting it as a \textit{forward-backward stochastic differential equation (FBSDE)}; see Appendix~\ref{app:bsde} for a primer on BSDEs.

We start by defining 
\begin{align}
\begin{cases}
X_t^\star &= \barcX_t, \quad Y_t^\star = u(t,\barcX_t),  \\ 
Z_t^\star &= \barbeta(t)\nabla u(t,\barcX_t).  
\end{cases}\label{eq:X_star}
\end{align}
Under Assumption~\ref{ass:smoothness}, a direct application of It\^o's formula and the use of Equation~\eqref{eq:uPDE} give us
\begin{align*}
dY_t^\star &= \partial_t u(t,X_t^\star)dt + \frac{1}{2} \Delta u(t,X_t^\star)  d\langle X^\star \rangle_t + \nabla u(t,X_t^\star) \cdot dX_t^\star \\
&= \Big(\partial_t u(t,X_t^\star) + \frac{1}{2} \Delta u(t,X_t^\star) \barbeta^2(t) + \nabla u(t,X_t^\star) \cdot \frac{1}{2} \barbeta^2(t) X_t^\star + \barbeta^2(t) \nabla u(t,X_t^\star)\Big)dt \\
&+ \barbeta(t) \nabla u(t,X_t^\star) \cdot d\overline{W}_t =\frac{1}{2} \|Z^\star_t\|^2 dt + Z_t^\star \cdot d\overline{W}_t,
\end{align*}
where the reader is reminded that $\overline{W}$ is the Brownian motion related to the reverse diffusion.
Therefore, $(X_t^\star, Y_t^\star, Z_t^\star)$ solves the FBSDE:
\begin{align}\label{eq:bsde}
\begin{cases}
d X_t = (\tfrac{1}{2}\barbeta^2(t) X_t + \barbeta(t) Z_t) dt + \barbeta(t) d\overline{W}_t, \\
d Y_t =  \frac{1}{2} \|Z_t\|^2 dt + Z_t \cdot d\overline{W}_t, \\ 
Y_T = \log \pi(X_T), \mbox{ and } X_0 = \barcX_0 = \cX_T.
\end{cases}    
\end{align}
The question that arises now is whether any solution of \eqref{eq:bsde} is a solution of the reverse SDE \eqref{eq:reverse}, giving uniqueness of solution for this FBSDE. We discuss this question in detail below. A positive answer allows us to remove the need to pre-estimate the score function $\nabla \log p$ prior to simulating the dynamics of $X$. Consequently, sampling from $p_0$ in the exact continuous-time formulation is equivalent to solving the associated FBSDE with the correct initial law $X_0\sim \mathcal L(\mathcal X_T)$. 

In practice this law is replaced by a Gaussian approximation; the resulting finite-time initialization bias is made explicit in Theorem~\ref{thm:tv-z-chen}. Moreover, the FBSDE formulation avoids the need to specify a domain over which the score is approximated, although numerical training still benefits from a broad domain.

We are then ready to state our main result, which supports the use of the FBSDE in \eqref{eq:bsde} to generate samples from $p_0$.
\begin{theorem}\label{thm:bsde}
Suppose Assumptions~\ref{ass:integrability}--\ref{ass:lipschitz-control} hold. Let $\mathcal F_t=\sigma(\barcX_0,\overline W_s:0\leq s\leq t)$, completed in the usual way. Then the FBSDE \eqref{eq:bsde}, initialized with $X_0=\barcX_0=\mathcal X_T$, admits a unique strong solution adapted to $(\mathcal F_t)_{0\leq t\leq T}$ among solutions $(X,Y,Z)$ for which
\[
\eta_t=\frac12\left(\bar\beta(t)\nabla u(t,X_t)-Z_t\right)
\]
satisfies 
\begin{align}
\begin{cases}
\displaystyle \mathbb{E}\left[e^{\tfrac{1}{2}\int_0^T \|\eta_t\|^2dt}\right] < +\infty \quad \text{(Novikov condition);}\\[10pt]
\displaystyle \mathbb{E}\left[e^{\int_0^T \eta_t \cdot d\overline{W}_t - \tfrac{1}{2}\int_0^T \|\eta_t\|^2dt} \int_0^T \|\eta_t\|^2dt \right] < +\infty.
\end{cases}\label{eq:admissibility}
\end{align}
Therefore, the unique solution is given by \((X_t^\star, Y_t^\star, Z_t^\star)_{t \in [0,T]}\) defined in \eqref{eq:X_star} and, in particular, $X_T^\star \sim p_0$.
\end{theorem}


Proving uniqueness under the standard theory of BSDEs is not straightforward, as the generator function exhibits quadratic growth and the terminal condition is possibly unbounded. For further non-standard results on uniqueness of BSDEs, see \cite{Uniqueness1, Uniqueness2, Uniqueness3}. To maintain the flow of the paper, we present the proof for Theorem~\ref{thm:bsde} in Appendix~\ref{app:uniqueness}.


\section{The algorithm}

Notice that, by Theorem~\ref{thm:bsde}, $X_T^\star = \barcX_T^\star = \cX_0^\star \sim p_0$ for the exact FBSDE initialized with $X_0^\star\sim\mathcal L(\mathcal X_T^\star)$. The implementable sampler replaces this law by $N(0,I_n)$. Therefore, to explore our FBSDE formulation of the reverse SDE given in
Equation~\eqref{eq:bsde}, we need to discuss the simulation of fully coupled
FBSDEs. 

The dependence of the drift of \(X\) on the backward \(Z\) process
makes the simulation of this equation particularly challenging, since the
forward paths cannot be generated independently of the backward component.
There are several numerical methods for this general class of equations; see,
for instance, the survey in \cite{surveyBSDE}. The present FBSDE has
additional features that make standard approaches difficult to apply directly.
In particular, the backward driver has quadratic growth in \(Z\), so
Picard-type schemes designed for Lipschitz generators or weakly coupled systems
are not directly applicable without additional truncation, localization, or
structural assumptions, \cite{BenderZhang2008,ImkellerDosReisZhang2010,Richou2011}. One might also try to use the exponential transform
commonly employed for quadratic BSDEs, but in the present setting this would
replace the terminal condition \(Y_T=\log \pi(X_T)\) by
\(\exp(-Y_T)=1/\pi(X_T)\), which is numerically undesirable because it
amplifies low-density regions and is less stable than working with the
unnormalized log-density. These considerations motivate a deep-learning
approximation of the $Y$ and $Z$ processes.

The numerical design follows two lessons from recent deep FBSDE practice. First, for coupled FBSDEs, a small training loss alone does not provide a reliable guarantee of sampling accuracy; validation should be performed on fresh simulated trajectories, as emphasized in convergence analyses of deep BSDE/FBSDE methods \cite{han2020convergence}. Second, the Schrödinger bridge FBSDE literature treats forward and backward processes as coupled objects and monitors whether the learned dynamics match the prescribed endpoint laws \cite{chen2022likelihood}. 

Our implementation is inspired by Deep BSDE methods, but it does not use the classical terminal-only loss from \cite{DeepBSDE}. Instead, it learns a time-dependent approximation of the value process and enforces the BSDE identity through one-step and short rollout residuals, \cite{hure2020deep, germain2022approximation}. Let $0=t_0<\cdots<t_N=T$ be a uniform partition of $[0,T]$, with $\Delta t = T/N$ and $\barbeta_{t_i}=\beta(T-t_i)$. We parametrize $Y_t = u(t,X_t)$ by a neural network $u_\theta(t,x)$ imposing the terminal condition structurally through
\begin{align}
u_\theta(t,x) = \log \pi(x) + (T-t)\,\NN_\theta(t,x),
\end{align}
where $\NN_\theta$ is the output of a neural network with weights and biases $\theta$. Hence $u_\theta(T,x)=\log \pi(x)$ exactly at terminal time for every $x$. The network takes as inputs the state $x$, and we add several features of the normalized time $t/T$ and of $x$. 

We approximate $Z$ using the parametrization suggested by Theorem~\ref{thm:bsde}:
\begin{align}
z_\theta(t,x)=
\barbeta(t)\,\nabla_x u_\theta(t,x).
\end{align}
Because the backward driver has quadratic growth in \(Z\), large values of
its approximation may destabilize both the simulated dynamics and the
residual loss. Projection-based truncations of the \(Z\)-argument are used in
numerical schemes for quadratic BSDEs; see
\cite{ChassagneuxRichou2016}. A related
truncated formulation is also considered in the analysis of a Deep BSDE
method for quadratic Hamilton--Jacobi--Bellman equations; see
\cite{ZhengEtAl2026}.

Motivated by these truncation approaches, we apply the smooth component-wise
transformation
\begin{align}
\widehat z_\theta(t,x)= c \,\tanh\!\left(\frac{z_\theta(t,x)}{c}\right),
\end{align}
where the clipping level $c$ is updated adaptively from a high empirical quantile of the current $z_\theta$ magnitudes.

Unlike gradient clipping in neural-network optimization, this transformation
acts directly on the estimated BSDE $Z$ process, rather than on
gradients with respect to the network parameters. The particular smooth
truncation and the adaptive quantile-based selection of \(c\) are
implementation choices of the present work.

Given a time-state pair $(\tau,\chi)$ and a Brownian increment $\Delta W \sim N(0,\Delta t\,I_n)$, we sample $X$ using an Euler step
\begin{align}\label{eq:iteration}
X
={}& \chi
+ \left(\frac{1}{2}\barbeta^2(\tau)\chi
+ \barbeta(\tau)\widehat z_\theta(\tau,\chi)\right)\Delta t + \barbeta(\tau)\Delta W,
\end{align}
and define the one-step BSDE residual
\begin{align}
&\mathcal R(\theta;\tau,\chi,\Delta W)
= u_\theta(\tau + \Delta t,X)
 - u_\theta(\tau,\chi) - \frac{1}{2}\|\widehat z_\theta(\tau,\chi)\|^2 \Delta t
 - \widehat z_\theta(\tau,\chi)\cdot \Delta W .
\end{align}
The one-step residual gives the local loss
\begin{align}
\mathcal L_{\mathrm{local}}(\theta)
=\frac{1}{m}\sum_{i=1}^m
\left|\mathcal R(\theta;\tau_i,\chi_i,\Delta W_i)\right|^2,
\end{align}
where $\{\Delta W_i\}_{i=1}^m$ is an iid sample of $N(0,\Delta t\,I_n)$ and $\{\tau_i,\chi_i\}_{i=1}^m$ is an iid sample of the distribution $q_{\mathrm{train}}$ defined in Section~\ref{sec:sampling}.

The implementation also includes a short rollout residual, which enforces temporal consistency of the learned BSDE dynamics by comparing, over a short simulated path segment, the change in \(u_\theta\) with the accumulated drift and martingale terms prescribed by the BSDE. More precisely, starting from sampled pairs $(t_j,X_j)$, we roll the same Euler dynamics forward for $\ell$ steps and penalize the discrepancy between the increment of $u_\theta$ and the accumulated BSDE driver:
\[
\begin{aligned}
&\mathcal L_{\mathrm{roll}}(\theta)
=
\mathbb E\Bigg[
\Bigg|
u_\theta(t_{j+\ell},X_{j+\ell})-u_\theta(t_j,X_j)-\sum_{r=j}^{j+\ell-1}
\left(
\frac12\|\widehat z_\theta(t_r,X_r)\|^2\Delta t
+\widehat z_\theta(t_r,X_r)\cdot\Delta W_r
\right)
\Bigg|^2
\Bigg].
\end{aligned}
\]
The total optimized loss is
\[
\mathcal L
=\mathcal L_{\mathrm{local}}
+\lambda_{\mathrm{roll}}\mathcal L_{\mathrm{roll}},
\]
where the rollout weight $\lambda_{\mathrm{roll}}$ is fixed a priori and kept constant across experiments in which the rollout loss is active, while setting $\lambda_{\mathrm{roll}}=0$ removes this term.

The minimization is performed using AdamW \cite{loshchilov2019decoupled}, an adaptive stochastic gradient method that extends Adam by applying weight decay separately from the gradient-based parameter update, improving the behavior of regularization in adaptive optimization. Let $\theta_k$ denote the network parameters after the $k$th optimization step. During training, we maintain an exponential moving average (EMA) initialized by $\bar{\theta}_0=\theta_0$ and updated according to $\bar{\theta}_k=\rho\bar{\theta}_{k-1}+(1-\rho)\theta_k$, where $\rho\in(0,1)$ controls the degree of smoothing. After the final iteration $K$, the averaged parameters $\bar{\theta}_K$, rather than the raw final iterate $\theta_K$, are used for evaluation and sample generation, reducing sensitivity to stochastic fluctuations and yielding more stable learned reverse dynamics.

We therefore regard $\mathcal L(\theta)$ as a training criterion rather than as a certificate of sampling accuracy. In addition to the one-step residual, the method is evaluated with trajectory-level quantities. One such diagnostic is the pathwise BSDE residual
\[
\begin{aligned}
&\mathcal L_{\mathrm{path}}(\theta)
=
\mathbb E\Bigg[
\Bigg|
\log \pi(X_N)-u_\theta(0,X_0)-\sum_{i=0}^{N-1}
\left(
\frac12\|\widehat z_\theta(t_i,X_i)\|^2\Delta t
+\widehat z_\theta(t_i,X_i)\cdot\Delta W_i
\right)
\Bigg|^2
\Bigg].
\end{aligned}
\]

We discuss the mini-batch sampling scheme in the next section.

\subsection{Training sampling scheme}\label{sec:sampling}

The time-state pairs used in $\mathcal L$ mix two sources for sampling: broadly distributed states, obtained from a mixture of a Gaussian and a uniform proposal on the computational domain, and on-policy states, obtained by periodically rolling out the current model and storing intermediate states along the trajectory. This replay mechanism helps place training mass on regions that are actually visited by the learned sampler. Figure~\ref{fig:training_sampling_scheme_visual} presents a diagram of the sampling method.

More precisely, let $D \subset \bR^n$ denote the rectangular computational domain used for a given target. The broad proposal used during training is a mixture of a centered Gaussian and a Uniform distribution on $D$
\begin{align}
q_{\mathrm{broad}}(x)
= \rho\,N(x; \, 0, \, \sigma^2_{\mathrm{broad}}I_n) + (1-\rho)\,\mathrm{Unif}_D(x),
\end{align}
with $\rho \in [0,1]$ fixed in advance. A discrete time index $i \in \{0,\dots,N-1\}$ is then sampled uniformly and paired with $q_{\mathrm{broad}}$.

In parallel, the implementation builds an on-policy replay distribution by rolling out the current model. Starting from
\begin{align}
\widetilde X_0 \sim N(0,I_n),
\end{align}
it recursively generates
\begin{align}
\widetilde X_{i+1}
={}& \widetilde X_i
+ \left(\frac{1}{2}\bar\beta_{t_i}^2\widetilde X_i
+ \bar\beta_{t_i}\widehat z_\theta(t_i,\widetilde X_i)\right)\Delta t
+ \bar\beta_{t_i}\Delta \widetilde{W}_i,
\end{align}
with $\Delta \widetilde{W}_i \stackrel{iid}{\sim} N(0,\Delta t\,I_n)$, and stores the intermediate pairs $(t_i,\widetilde X_i)$. The actual training distribution is therefore a mixture of the broad proposal and the empirical replay measure:
\begin{align*}
q_{\mathrm{train}}
&= (1-\eta)\,q_{\mathrm{broad}}\otimes \mathrm{Unif}\{t_0,\dots,t_{N-1}\}
+ \eta\,q_{\mathrm{replay}},
\end{align*}
where $\eta \in [0,1]$ denotes the prescribed on-policy fraction.

\begin{figure}[t]
\centering
\resizebox{\textwidth}{!}{%
\begin{tikzpicture}[
  font=\sffamily\scriptsize,
  >=Latex,
  line cap=round,
  line join=round,
  box/.style={
    draw=#1!72,
    fill=#1!6,
    rounded corners=2.6pt,
    line width=0.55pt,
    align=center,
    inner xsep=5pt,
    inner ysep=4.5pt,
    minimum height=12mm,
    text width=33mm
  },
  box/.default=rbgray,
  mini/.style={
    draw=#1!72,
    fill=#1!6,
    rounded corners=2.4pt,
    line width=0.5pt,
    align=center,
    inner xsep=4pt,
    inner ysep=3.8pt,
    minimum height=9.5mm,
    text width=29mm
  },
  mini/.default=rbgray,
  panel/.style={
    draw=black!18,
    fill=white,
    rounded corners=6pt,
    line width=0.45pt,
    inner xsep=12pt,
    inner ysep=15pt
  },
  arrow/.style={->, line width=0.58pt, draw=black!70},
  softarrow/.style={->, line width=0.5pt, draw=black!50, dashed},
  lab/.style={font=\sffamily\tiny, text=black!70},
  alabel/.style={font=\sffamily\tiny, fill=white, inner sep=1.3pt, text=black!72},
  stage/.style={font=\sffamily\bfseries\scriptsize, text=black!78, fill=white, inner sep=1.5pt}
]

\node[stage, anchor=west] (title_time) at (-0.15,3.75) {1. Choose a time};

\node[mini=rbblue, anchor=west, text width=31mm] (time) at (0.40,2.95)
{sample $i\sim\mathrm{Unif}\{0,\ldots,N-1\}$\\[-0.2mm]
set $\tau=t_i$};

\coordinate (timepanelSW) at (-0.15,2.35);
\coordinate (timepanelNE) at (4.25,4.00);

\node[stage, anchor=west] (title_broad) at (-0.15,1.35) {2a. Broad-coverage state};

\node[box=rbgreen, anchor=west] (broadlaw) at (0,0.55)
{\textbf{Broad proposal}\\[-0.3mm]
$q_{\rm broad}=\rho N(0, \, \sigma^2_{\mathrm{broad}}I_n)+(1-\rho)\mathrm{Unif}_D$};

\begin{scope}[shift={(3.95,-0.18)}]
  \node[lab] at (1.45,1.60) {state sample in $D$};
  \draw[rounded corners=2.8pt, draw=rbgreen!65, fill=white, line width=0.55pt]
  (0,0) rectangle (2.9,1.38);
  \node[lab] at (0.23,1.17) {$D$};
  \draw[draw=rbgreen!45, fill=rbgreen!5, line width=0.4pt]
  (1.45,0.72) ellipse (0.70 and 0.36);
  \foreach \p in {(0.30,0.33),(0.58,0.95),(0.86,0.50),(1.10,1.07),(1.36,0.70),(1.60,0.83),(1.86,0.53),(2.13,0.90),(2.43,0.38),(2.62,1.02)}{
    \fill[rbgreen!78] \p circle (1.15pt);
  }
  \node[lab] at (1.45,-0.25) {$x^{(b)}\sim q_{\rm broad}$};
\end{scope}

\coordinate (broadpicSW) at (3.90,-0.72);
\coordinate (broadpicNE) at (6.95,1.62);

\node[mini=rborange, anchor=west] (broadpair) at (7.35,0.55)
{\textbf{broad pair}\\[-0.2mm]
$(\tau,\chi)=(t_i,x^{(b)})$};

\draw[arrow] (broadlaw.east) -- (3.95,0.55);
\draw[arrow] (6.85,0.55) -- (broadpair.west);
\draw[softarrow] (time.south east) .. controls +(0.55,-0.35) and +(0.75,0.50) .. (broadpair.north);

\node[stage, anchor=west] (title_replay) at (-0.15,-1.45) {2b. On-policy replay state};

\node[box=rbpurple, anchor=west] (rollout) at (0,-2.30)
{\textbf{Roll out current model}\\[-0.3mm]
$\widetilde X_0\sim N(0,I_n)$\\[-0.3mm]
Euler steps using $\widehat z_\theta$};

\begin{scope}[shift={(3.95,-3.20)}]
  \node[lab] at (1.62,1.75) {stored trajectory};
  \draw[rounded corners=2.8pt, draw=rbpurple!65, fill=white, line width=0.55pt]
  (0,0) rectangle (3.25,1.55);
  \draw[line width=0.78pt, rbpurple!78]
    (0.28,0.32) .. controls (0.65,1.08) and (1.00,0.23) .. (1.28,0.84)
                 .. controls (1.65,1.23) and (1.92,0.33) .. (2.25,0.78)
                 .. controls (2.60,1.10) and (2.85,0.70) .. (3.00,1.12);
  \foreach \p in {(0.28,0.32),(0.78,0.80),(1.28,0.84),(2.25,0.78),(3.00,1.12)}{
    \fill[rbpurple!82] \p circle (1.3pt);
  }
  \node[lab] at (0.28,0.10) {$\widetilde X_0$};
  \node[lab] at (0.78,1.06) {$\widetilde X_1$};
  \node[lab] at (1.28,0.58) {$\widetilde X_2$};
  \node[lab] at (2.25,0.52) {$\widetilde X_j$};
  \node[lab] at (3.00,1.36) {$\widetilde X_N$};
  \draw[line width=0.42pt, black!55] (0.20,0.13) -- (3.07,0.13);
  \foreach \x/\labt in {0.28/t_0,0.78/t_1,1.28/t_2,2.25/t_j,3.00/t_N}{
    \draw[line width=0.42pt, black!55] (\x,0.06) -- (\x,0.20);
    \node[below=0.6mm, font=\sffamily\tiny] at (\x,0.06) {$\labt$};
  }
\end{scope}

\coordinate (replaypicSW) at (3.90,-3.92);
\coordinate (replaypicNE) at (7.25,-1.25);

\node[mini=rborange, anchor=west] (replaypair) at (7.35,-2.30)
{\textbf{replay pair}\\[-0.2mm]
$(\tau,\chi)=(t_j,\widetilde X_j)$};

\draw[arrow] (rollout.east) -- (3.95,-2.30);
\draw[arrow] (7.20,-2.30) -- (replaypair.west);

\node[stage, anchor=west] (title_mix) at (10.90,1.35) {3. Mix sources and form mini-batch};

\node[draw=rbred!75, fill=rbred!7, diamond, aspect=1.65, line width=0.55pt,
      inner sep=1.5pt, align=center, text width=17mm] (mix) at (11.75,-0.82)
{mix\\[-0.5mm]sources};

\draw[arrow] (broadpair.east) -- node[alabel, above] {$1-\eta$} (mix.north west);
\draw[arrow] (replaypair.east) -- node[alabel, below] {$\eta$} (mix.south west);

\node[box=rbred, anchor=west, text width=38mm] (qtrain) at (13.35,-0.82)
{\textbf{Training law}\\[-0.3mm]
$q_{\rm train}=(1-\eta)q_{\rm broad}\otimes\mathrm{Unif}\{t_k\}+\eta q_{\rm replay}$};

\draw[arrow] (mix.east) -- (qtrain.west);

\node[box=rborange, anchor=west, text width=37mm] (batch) at (18.20,-0.82)
{\textbf{Mini-batch}\\[-0.2mm]
$\{(\tau_r,\chi_r)\}_{r=1}^m\sim q_{\rm train}$\\[-0.2mm]
$(t_{i_1},x^{(b)}_1),\ (t_{j_2},\widetilde X_{j_2}),\ldots$};

\draw[arrow] (qtrain.east) -- (batch.west);

\node[mini=rbgray, anchor=west, text width=34mm] (loss) at (18.20,-2.85)
{used to build the local and rollout BSDE residuals};

\draw[arrow] (batch.south) -- (loss.north);

\coordinate (refreshRight) at ($(loss.south)+(0,-0.65)$);
\coordinate (refreshLeft)  at (rollout.south |- refreshRight);

\draw[softarrow]
  (loss.south) -- (refreshRight)
  -- node[alabel, below, pos=0.50] {periodic replay refresh}
     (refreshLeft)
  -- (rollout.south);

\begin{pgfonlayer}{background}
  \node[panel, fit=(title_time)(time)(timepanelSW)(timepanelNE)] {};
  \node[panel, fit=(title_broad)(broadlaw)(broadpair)(broadpicSW)(broadpicNE)] {};
  \node[panel, fit=(title_replay)(rollout)(replaypair)(replaypicSW)(replaypicNE)] {};
  \node[panel, fit=(title_mix)(mix)(qtrain)(batch)(loss)] {};
\end{pgfonlayer}

\end{tikzpicture}%
}
\caption{Visual representation of the training sampling scheme. A time index is sampled and paired either with a broad proposal state \(x^{(b)}\sim q_{\rm broad}\), producing \((\tau,\chi)=(t_i,x^{(b)})\), or with an on-policy replay state \(\widetilde X_j\), producing \((\tau,\chi)=(t_j,\widetilde X_j)\). The two sources are mixed with weights \(1-\eta\) and \(\eta\), respectively, to form \(q_{\rm train}\). Mini-batches \(\{(\tau_r,\chi_r)\}_{r=1}^m\sim q_{\rm train}\) are then used to compute the local and rollout BSDE residuals.}
\label{fig:training_sampling_scheme_visual}
\end{figure}

\subsection{Direct and hybrid parametrizations of \texorpdfstring{$Z$}{Z}}
\label{sec:cons_Z}

In addition to the gradient parametrization
\begin{equation}
z_\theta^{\mathrm{grad}}(t,x)
=
\bar{\beta}(t)\nabla_x u_\theta(t,x),
\end{equation}
we consider a direct parametrization of the $Z$ process:
\begin{equation}
z_\theta^{\mathrm{dir}}(t,x)
=
\bar{\beta}(t)v_\theta(t,x),
\end{equation}
where \(v_\theta\) is represented by a separate neural network. This
parametrization relaxes the structural identity between the $Y$ and $Z$ processes and may therefore reduce the optimization burden associated
with differentiating \(u_\theta\). The discrepancy
\[
z_\theta^{\mathrm{dir}}(t,x)
-
\bar{\beta}(t)\nabla_x u_\theta(t,x)
\]
can be monitored as a diagnostic or explicitly penalized during training.

For more challenging targets, we also consider a hybrid parametrization that
incorporates the known terminal control. Since
\begin{equation}
Z_T^\star
=
\bar{\beta}(T)\nabla\log\pi(X_T^\star),
\end{equation}
we define
\begin{align}
z_\theta^{\mathrm{hyb}}(t,x)
={}&
\alpha(t)\bar{\beta}(t)\nabla\log\pi(x)
+
\bigl(1-\alpha(t)\bigr)
z_\theta^{\mathrm{dir}}(t,x),
\end{align}
where
\(\alpha(t)
=
\left(t/T\right)^p\), $p > 0$. Thus, the directly learned $Z$ process dominates at early times, whereas the
known terminal control is imposed as \(t\) approaches \(T\). For \(t<T\),
the term involving \(\nabla\log\pi\) should be interpreted as a
boundary-informed stabilization rather than as the exact time-dependent
score.

When the hybrid parametrization is used, we encourage consistency with the
FBSDE identity through the penalty
\begin{align}
\mathcal L_{\mathrm{cons}}(\theta)
:=
\frac{1}{m}
\sum_{i=1}^{m}
\Bigl\|
z_\theta^{\mathrm{hyb}}(\tau_i,\chi_i)
-
\bar{\beta}(\tau_i)
\nabla_x u_\theta(\tau_i,\chi_i)
\Bigr\|^2.
\label{eq:loss-consistency}
\end{align}
The corresponding training objective becomes
\begin{equation}
\mathcal L(\theta)
=
\mathcal L_{\mathrm{local}}(\theta)
+
\lambda_{\mathrm{roll}}
\mathcal L_{\mathrm{roll}}(\theta)
+
\lambda_{\mathrm{cons}}
\mathcal L_{\mathrm{cons}}(\theta).
\end{equation}

For the gradient parametrization, the consistency identity holds by
construction, so no additional penalty is required. For the unconstrained
direct parametrization, the discrepancy may be used only as a diagnostic,
which corresponds to setting \(\lambda_{\mathrm{cons}}=0\). For the hybrid
parametrization, we take \(\lambda_{\mathrm{cons}}>0\), chosen a priori, retaining the
flexibility of a directly learned $Z$ process while discouraging departures
from the FBSDE relation.

The full algorithm is presented below and illustrated in Figure~\ref{fig:diagram}.

\begin{algorithm}[H]
\caption{Reverse-BSDE Diffusion Sampler (RBDS)}
\label{alg:local-bsde}
\begin{algorithmic}[1]
\Require target log-density $\log \pi$, horizon $T$, diffusion scale $\beta$, grid $0=t_0<\cdots<t_N=T$, training steps $K$, batch size $m$
\State Initialize network parameters $\theta$
\For{$k=1,\ldots,K$}
    \State Sample a mini-batch of state-time pairs $(\tau_i,\chi_i)$ from the training distribution
    \State Evaluate the parametrized $u_\theta(\tau_i,\chi_i)$ and $z_\theta(\tau_i,\chi_i)$
    \State Compute the local residual loss, the rollout residual loss, and, if active, the consistency loss
    \State Update $\theta$ with AdamW
\EndFor
\State Sample $X_0 \sim N(0,I_n)$
\For{$i=0,\ldots,N-1$}
    \State Evaluate the trained $Z$ and propagate one Euler step of the reverse diffusion
\EndFor
\State \Return terminal samples $\widehat X_N$
\end{algorithmic}
\end{algorithm}

\begin{figure}
    \centering
    \resizebox{\textwidth}{!}{
    \begin{tikzpicture}[
  font=\sffamily\scriptsize,
  >=Latex,
  line cap=round,
  line join=round,
  block/.style={
    draw=#1!80,
    fill=#1!6,
    rounded corners=2.5pt,
    line width=0.55pt,
    align=center,
    inner xsep=5pt,
    inner ysep=4.7pt,
    minimum height=14.5mm,
    text width=32mm
  },
  block/.default=gray,
  slim/.style={
    draw=#1!80,
    fill=#1!6,
    rounded corners=2.5pt,
    line width=0.55pt,
    align=center,
    inner xsep=4.2pt,
    inner ysep=4pt,
    minimum height=10.5mm,
    text width=30mm
  },
  slim/.default=gray,
  tag/.style={
    circle,
    draw=#1!80,
    fill=white,
    line width=0.5pt,
    minimum size=4.6mm,
    inner sep=0pt,
    font=\sffamily\bfseries\tiny,
    text=#1!70!black
  },
  tag/.default=gray,
  band/.style={
    draw=black!23,
    fill=black!1,
    rounded corners=5pt,
    inner sep=7pt
  },
  arrow/.style={->, line width=0.55pt, draw=black!68},
  softarrow/.style={->, line width=0.5pt, draw=black!55, dashed},
  label/.style={font=\sffamily\tiny, fill=white, inner sep=1.2pt, text=black!70}
]

\definecolor{rbblue}{RGB}{52,104,176}
\definecolor{rbgreen}{RGB}{55,145,110}
\definecolor{rborange}{RGB}{206,128,44}
\definecolor{rbpurple}{RGB}{118,88,164}
\definecolor{rbred}{RGB}{185,75,75}
\definecolor{rbgray}{RGB}{80,80,80}

\node[block=rbblue] (setup) at (0,0) {\textbf{Problem setup}\\[-0.4mm]
$\log\pi$, $T$, $\beta$, grid\\
$0=t_0<\cdots<t_N=T$};
\node[tag=rbblue, anchor=north west] at ($(setup.north west)+(-1.4mm,1.4mm)$) {1};

\node[block=rbgreen] (pairs) at (4.35,0) {\textbf{Training pairs}\\[-0.4mm]
$(\tau,\chi)\sim q_{\rm train}$\\
broad proposal $+$ replay};
\node[tag=rbgreen, anchor=north west] at ($(pairs.north west)+(-1.4mm,1.4mm)$) {2};

\node[block=rborange] (network) at (8.70,0) {\textbf{Evaluate networks}\\[-0.4mm]
$u_\theta=\log\pi+(T-t)\NN_\theta$\\
$\widehat z_\theta=c\tanh(z_\theta/c)$};
\node[tag=rborange, anchor=north west] at ($(network.north west)+(-1.4mm,1.4mm)$) {3};

\node[block=rbpurple] (euler) at (13.05,0) {\textbf{Euler FBSDE step}\\[-0.4mm]
$\Delta W\sim N(0,\Delta t I)$\\
$X'\leftarrow\chi+b_\theta\Delta t+\bar\beta\Delta W$};
\node[tag=rbpurple, anchor=north west] at ($(euler.north west)+(-1.4mm,1.4mm)$) {4};

\node[block=rbred] (loss) at (17.40,0) {\textbf{Residual objective}\\[-0.4mm]
$\mathcal L_{\rm local}+\lambda_{\rm roll}\mathcal L_{\rm roll}$\\
$(+\lambda_{\rm cons}\mathcal L_{\rm cons})$};
\node[tag=rbred, anchor=north west] at ($(loss.north west)+(-1.4mm,1.4mm)$) {5};

\draw[arrow] (setup) -- (pairs);
\draw[arrow] (pairs) -- (network);
\draw[arrow] (network) -- (euler);
\draw[arrow] (euler) -- (loss);

\node[slim=rbgreen] (replay) at (4.35,-2.25) {\textbf{Replay refresh}\\[-0.4mm]
roll out current model;\\ store $(t_i,\widetilde X_i)$};
\node[slim=rbred] (adam) at (8.70,-2.25) {\textbf{AdamW update}\\[-0.4mm]
$\theta\leftarrow\theta-\alpha\widehat\nabla_\theta\mathcal L$};

\draw[arrow, rounded corners=5pt] (loss.south) -- ++(0,-0.55) -| (adam.east);
\draw[arrow] (adam.north) -- node[label, right] {updated $\theta$} (network.south);
\draw[softarrow] (adam.west) -- node[label, above] {periodically} (replay.east);
\draw[softarrow] (replay.north) -- node[label, left] {on-policy states} (pairs.south);

\begin{scope}[on background layer]
\node[band, fit=(setup)(pairs)(network)(euler)(loss)(adam)(replay)] (trainband) {};
\node[anchor=south west, font=\sffamily\small\bfseries, text=black!75, fill=white, inner sep=1pt] at ($(trainband.north west)+(1mm,1mm)$) {Training stage};
\end{scope}

\node[block=rbblue] (init) at (4.35,-5.60) {\textbf{Initialize}\\[-0.4mm]
$\widehat X_0\sim N(0,I_n)$};
\node[slim=rborange] (trained) at (8.70,-5.60) {\textbf{Trained $\widehat Z$}\\[-0.4mm]
 $\widehat z_\theta(t,x)$};
\node[block=rbpurple] (rollout) at (13.05,-5.60) {\textbf{Reverse SDE simulation}\\[-0.4mm]
for $i=0,\ldots,N-1$\\
Euler step with $\widehat z_\theta$};
\node[block=rbred] (output) at (17.40,-5.60) {\textbf{Return samples}\\[-0.4mm]
terminal particles\\
$\widehat X_N\approx p_0$};
\node[slim=rbgray] (diag) at (17.40,-7.55) {\textbf{Validation diagnostics}\\[-0.4mm]
path residual, covariance,\\ marginals, geometry};

\draw[arrow] (init) -- (trained);
\draw[arrow] (trained) -- (rollout);
\draw[arrow] (rollout) -- (output);
\draw[softarrow] (output) -- (diag);
\draw[softarrow] (adam.south) -- node[label, right] {after training} (trained.north);

\begin{scope}[on background layer]
\node[band, fit=(init)(trained)(rollout)(output)(diag)] (sampleband) {};
\node[anchor=south west, font=\sffamily\small\bfseries, text=black!75, fill=white, inner sep=1pt] at ($(sampleband.north west)+(1mm,1mm)$) {Sampling stage};
\end{scope}

\end{tikzpicture}}
    \caption{Schematic representation of the Reverse-BSDE Diffusion Sampler. The training stage samples time--state pairs from a mixture of broad proposal states and on-policy replay states, evaluates the parametrized functions \(u_\theta\) and  \(\widehat z_\theta\), propagates one Euler step of the coupled FBSDE, and updates the network parameters by minimizing local and rollout BSDE residuals. After training, the learned reverse dynamics are initialized from \(N(0,I_n)\) and simulated forward in reverse time to produce terminal particles \(\widehat X_N\), which are used as approximate samples from the target distribution \(p_0\).}

    \label{fig:diagram}
\end{figure}

\subsection{Finite-time initialization and approximation error}

Theorem~\ref{thm:bsde} is an exact statement for the time-reversed process initialized at the true finite-time marginal distribution of $\cX_T$, $p(T,\cdot)$. This distinction is important because the numerical sampler starts from a standard Gaussian. More precisely, $p(T,\cdot)$ only converges to $N(0,I_n)$ as $T$ grows, but they are not identical for a fixed finite horizon. 

With some abuse of notation, let $X_N$ denote the Euler--Maruyama discretization of the
same dynamics on the grid \(0=t_0<\cdots<t_N=T\), again using the exact \(Z^\star\) process. We define the Euler discretization error by the total variation distance:
\[
\varepsilon_{\mathrm{disc}}(T)
:=
\operatorname{TV}(
\mathcal L(X_N),
\mathcal L(X_T^\star)
).
\]
Thus \(\varepsilon_{\mathrm{disc}}\) measures only the error caused by replacing
the continuous-time reverse dynamics by the Euler scheme; it is separate from
both the finite-time initialization error and the learned $Z$ error.

We then quantify the difference between the exact finite-time initial law and the Gaussian initialization used by the implementable sampler. We provide an upper bound for the approximation error when the exact $Z$ process is replaced by a learned $Z$ process and the dynamics are discretized in time. Its proof is given in Appendix~\ref{app:uniqueness} and it is a consequence of our FBSDE reformulation and the results in \cite{chen2023sampling}. The required assumptions are outlined below and are standard in the literature. 

To state the corresponding discretization condition in the present time-inhomogeneous OU setting, define the effective OU horizon and effective mesh size by, respectively,
\begin{align*}
T_\beta
:=
\frac12\int_0^T \beta^2(s)\,ds, \quad
h_\beta
:=
\max_{0\le i<N}
\frac12\int_{t_i}^{t_{i+1}}\bar\beta^2(s)\,ds .
\end{align*}

\begin{assumption}[Finite second moment]
\label{ass:second-moment}
\(
\int_{\mathbb R^n}
\|x\|^2 p_0(x)\,\mathrm dx
<\infty.
\)
\end{assumption}

\begin{assumption}[Finite relative entropy]
\label{ass:finite-kl}
\(
\operatorname{KL}
\bigl(p_0\,\Vert\,N(0,I_n)\bigr)<\infty.
\)
\end{assumption}

\begin{assumption}[Time discretization]
\label{ass:discretization}
The grid is sufficiently fine in operational OU time:
\[
h_\beta\lesssim\frac{1}{L_s+1},
\]
where $\lesssim$ means less than or equal to, up to a constant multiplicative factor; and $L_s$ is the Lipschitz constant from Assumption \ref{ass:lipschitz-control}.
\end{assumption}

\begin{assumption}[$Z$ approximation error]
\label{ass:control-error}
The time-integrated mean-squared error between the approximate and the exact $Z$ process, evaluated along trajectories of the exact process \(X^\star\), is finite:
\[
\varepsilon_Z^2(T)
:=
\int_0^T
\mathbb E\!\left[
  \left\lVert
    \widehat z(t,X_t^\star)-Z_t^\star
  \right\rVert^2
\right]
\,\mathrm dt < +\infty.
\]
\end{assumption}

\begin{remark}
For a uniform grid with step size \(\Delta t=T/N\), a sufficient condition for \ref{ass:discretization} is
\[
\Delta t
\lesssim
\frac{1}{(L_s+1)\|\beta\|_\infty^2}.
\]
\end{remark}

Under the corresponding hypotheses of \cite{chen2023sampling}, the Euler
discretization term is the discretization error after the OU time change
\(s(t)=\frac12\int_0^t\beta^2(r)\,dr\). In particular, up to universal
constants and logarithmic factors, one may write
\[
\frac{\varepsilon_{\mathrm{disc}}(T)}{ \sqrt{T_\beta \, h_\beta}(L_s+1)}
\lesssim
\sqrt n
+
\sqrt{h_\beta \, \int_{\mathbb R^n}
\|x\|^2 p_0(x)\,\mathrm dx}
.
\]

\begin{theorem}[Total-variation error in terms of the \(Z\) process]
\label{thm:tv-z-chen}
Under Assumptions~\ref{ass:integrability}--\ref{ass:control-error}, it holds
\begin{align*}
\operatorname{TV}
(
\mathcal L(\widehat X_N),
p_0
)
&\lesssim
e^{-T_\beta}
\sqrt{\operatorname{KL}(p_0\|N(0,I_n))}
+
\varepsilon_{\mathrm{disc}}(T)
+
\varepsilon_Z(T).
\end{align*}
\end{theorem}

Theorem~\ref{thm:tv-z-chen} should be interpreted as an error-decomposition result, not as a convergence theorem for the neural-network training procedure. The optimized residual losses used below do not, by themselves, prove that $\varepsilon_Z$ is small.

\section{Numerical experiments}
\label{sec:numerics}

This section evaluates the proposed Reverse-BSDE Diffusion Sampler (RBDS) on two-dimensional synthetic targets. The purpose is not to claim state-of-the-art sampling performance, but to understand which parts of the FBSDE training procedure are useful for different target geometries and to compare the resulting sampler with a standard Markov-chain baseline.

\begin{table}
\centering
\renewcommand{\arraystretch}{1.4}

\begin{tabularx}{\textwidth}{@{} l L @{}}
\toprule
\textbf{Target} & \textbf{Description} \\
\midrule

Quartic
& Light-tailed product target with density proportional to $\exp[-(x_1^4+x_2^4)]$. \\

Moderately separated mixture
& Three-component mixture with modes at $(-3,-3)$, $(0,0)$, and $(3,3)$. \\

Spaced mixture
& Three-component mixture with widely separated modes at $(-5,-5)$, $(0,0)$, and $(5,5)$. \\

Correlated Gaussian
& Gaussian target with strong linear dependence and correlation $0.95$. \\

Banana
& Nonlinear target obtained from a volume-preserving transform of a Gaussian law. \\

Rings
& Radially multimodal target with preferred radii $2$, $4$, $6$, and $8$. \\
\bottomrule
\end{tabularx}

\caption{Synthetic target distributions used in the numerical experiments.}
\label{tab:targets-summary}
\end{table}

\subsection{Synthetic targets}

We evaluate the method on six synthetic targets, chosen to cover multimodal, light-tailed, correlated, and geometrically structured distributions. The targets are summarized in Table~\ref{tab:targets-summary}.

The unnormalized densities used in the experiments are defined as follows. The two Gaussian mixture targets have equal weights and common covariance \(0.45I_2\):
\[
\pi_{\mathrm{mix}}(x)
\propto
\sum_{j=1}^3
\exp\left(-\frac{\|x-\mu_j\|^2}{2(0.45)}\right).
\]

The light-tailed quartic target is
\[
\pi_{\mathrm{quartic}}(x)
\propto
\exp\left(-(x_1^4+x_2^4)\right).
\]

The correlated Gaussian target is \(N(0,\Sigma)\), with
\[
\Sigma
=
\begin{pmatrix}
1 & 0.95 \\
0.95 & 1
\end{pmatrix}.
\]

For the banana-shaped target, let \(a=2\), \(b=0.2\), \(\mu=(0,4)\), and
\[
\Sigma_b
=
\begin{pmatrix}
1 & 0.5 \\
0.5 & 1
\end{pmatrix}.
\]
We use the volume-preserving transformation
\[
g(x)
=
\left(
\frac{x_1}{a},
\;
a x_2 + ab(x_1^2+a^2)
\right),
\]
and define
\[
\pi_{\mathrm{banana}}(x)
\propto
\exp\left(
-\frac12
\bigl(g(x)-\mu\bigr)^\top
\Sigma_b^{-1}
\bigl(g(x)-\mu\bigr)
\right).
\]

Finally, for the rings target, let \(r=\|x\|\). The radial density is a mixture of positive-truncated one-dimensional Gaussians with means
\[
(2,4,6,8),
\]
weights
\[
(0.1,0.4,0.1,0.4),
\]
and common standard deviation \(0.2\). The corresponding planar density is
\[
\pi_{\mathrm{rings}}(x)
\propto
\frac{f_R(r)}{2\pi r},
\]
with \(r\) truncated below at \(10^{-6}\) in the implementation.

\subsection{Implementation details}

The neural network is a time-conditioned feed-forward model with four hidden layers of width $96$, SiLU activations, and sinusoidal embeddings of the normalized time variable. The inputs are $x$, $t/T$, $\sin(\pi t/T)$, $\cos(\pi t/T)$, $\sin(2\pi t/T)$, and $\cos(2\pi t/T)$; the rings experiment additionally includes the radius $\|x\|$. The light-tailed, high-correlation, and banana examples use the gradient-based parametrization of $Z$; the two Gaussian-mixture examples use a direct $Z$ parametrization; and the rings example uses the hybrid parametrization; in this case, we set $\lambda_{\mathrm{cons}}=0.25$. 

The optimization is performed with a short linear warmup followed by cosine learning-rate decay, and an exponential moving average of the network weights for evaluation. The learning rate starts at $2\times 10^{-3}$ with weight decay of $10^{-6}$; the first 100 steps use a linear warmup from $10\%$ of the base learning rate, and the cosine schedule decays to $10^{-5}$. Unless otherwise stated, we use gradient clipping at norm $1$, $N=112$ time steps, and an exponential moving average with decay $0.995$ starting after step 100. The default loss weights are $\lambda_{\mathrm{roll}}=0.75$, and $\lambda_{\mathrm{cons}}=0.25$ when the consistency term is active. We generate $6,000$ samples per target for the main diagnostics, but plot $3,000$ of them for readability.

Moreover, for the mini-batch sampling, we use $\rho=0.5$, $\sigma_{\mathrm{broad}}=2.2$, $\eta=0.2$, and 384 on-policy paths. The replay buffer is refreshed every 50 optimization steps. The rectangular domain $D$ is the plotting and terminal-anchor domain associated with each target: $[-6.5,6.5]^2$ for the moderately separated mixture, $[-8.5,8.5]^2$ for the spaced mixture, $[-2.2,2.2]^2$ for the quartic target, $[-4.5,4.5]^2$ for the high-correlation Gaussian, $[-4.5,4.5]\times[-3.5,3.5]$ for the banana target, and $[-9.5,9.5]^2$ for the rings target.

The $Z$ process approximation is clipped component-wise by $c\tanh(z/c)$. The clipping level is initialized at $c=4$ and then updated adaptively: at each step we compute the 0.98 quantile of the absolute raw $Z$ process values, multiply it by $1.35$, constrain the result to $[1.5,20]$, and update $c$ with exponential smoothing factor $0.97$.

For the light-tailed, high-correlation, and banana targets, the baseline choice is $T=2.6$ and $\beta=1$. The spaced mixture uses $\beta=2$. The rings example uses $T=3$, a hybrid $Z$ representation, the additional radius feature $\|x\|$, and the consistency penalty $\mathcal L_{\mathrm{cons}}$ from Section~\ref{sec:cons_Z}. During model selection below, we also consider target-specific variants such as a longer horizon, removing the rollout term, or disabling $Z$ clipping.

The two Gaussian-mixture examples test multimodality at two different separation scales. The light-tailed quartic, high-correlation Gaussian, and banana targets probe respectively sharp concentration, anisotropy, and nonlinear geometry. The rings target is the most structured example in the paper and is also the only one that uses the hybrid $Z$ with the consistency term and the radius feature.

\subsection{Baseline: MALA}
\label{sec:mala}

As a reference Markov-chain baseline we use MALA, the Metropolis-adjusted Langevin algorithm, which only requires pointwise evaluation of $\log \pi$ up to an additive constant and its score $\nabla \log \pi$. Given a current state $x$, MALA proposes
\begin{align}
    y = x + \frac{h^2}{2}\nabla \log \pi(x) + h \xi,
    \quad \xi \sim N(0,I_2),
\end{align}
and accepts the proposal with the usual Metropolis--Hastings correction
\begin{align}
    \alpha(x,y)
    =
    \min\left\{
    1,
    \frac{\pi(y)q_h(x\mid y)}{\pi(x)q_h(y\mid x)}
    \right\},
\end{align}
where $q_h(\cdot\mid x)$ is the Gaussian proposal density induced by the Langevin step. We run independent MALA chains in parallel, using target-specific step sizes chosen by short pilot runs. Each MALA run uses $2,500$ transitions with a burn-in of $500$ transitions and returns the final states of $6,000$ chains.\footnote{The step sizes used for the mixture, spaced mixture, light-tailed, high-correlation, banana, and rings targets are respectively
\(
0.38, 0.28, 0.45, 0.16, 0.22, 0.14.
\)
The corresponding post-burn-in acceptance rates are approximately
\(
0.977, 0.991, 0.833, 0.971, 0.975, 0.973.
\)}
This comparison is deliberately favorable to MALA in low dimension: MALA uses the exact target score, whereas the RBDS method learns time-dependent $Y$ and $Z$ processes before sampling. Its weakness in these examples is a lack of global mixing across well-separated components or annular modes.

\subsection{Evaluation protocol and diagnostics}
\label{sec:evaluation-diagnostics}

The optimized residual loss is used to monitor numerical training, but it is
not treated as a certificate of sampling accuracy. As discussed above, a
small BSDE residual does not by itself imply that the terminal law of the
learned dynamics is close to \(p_0\). We therefore assess the generated
particles using distributional diagnostics computed against high-accuracy
reference quantities whenever these are available.

For our target distribution $p_0$, we report the covariance error
\begin{equation}
E_{\mathrm{cov}}
=
\left\|
\widehat{\Sigma}-\Sigma_{p_0}
\right\|_{\mathrm F},
\label{eq:empirical-covariance-error}
\end{equation}
where \(\widehat{\Sigma}\) is the empirical covariance matrix of the generated
sample and
\(\Sigma_{p_0}=\operatorname{Cov}_{p_0}(X)\). This diagnostic is sensitive to
discrepancies in marginal scale, dependence, and orientation.

We also report the average coordinate-wise Wasserstein discrepancy
\begin{equation}
W_{1,\mathrm{marg}}
=
\frac{1}{n}
\sum_{r=1}^{n}
W_1\!\left(\widehat F_r,F_{p_0,r}\right),
\label{eq:empirical-marginal-wasserstein}
\end{equation}
where \(\widehat F_r\) and \(F_{p_0,r}\) are, respectively, the empirical and
reference marginal distribution functions of the \(r\)th coordinate. The
normalization by \(n\) makes the scale of this diagnostic comparable across
dimensions. It assesses agreement of the coordinate marginals, but does not
by itself characterize the joint distribution.

For targets whose geometry is naturally described by distance from a
reference point \(x_{\mathrm c}\in\mathbb R^n\), we additionally report the
radial moment diagnostic
\begin{equation}
G_{\mathrm{rad}}
=
\left|
\widehat{\mathbb E}[\|X-x_{\mathrm c}\|]-\mathbb E_{p_0}[\|X-x_{\mathrm c}\|]
\right|
+
\left|
\widehat{\operatorname{sd}}(\|X-x_{\mathrm c}\|)
-
\operatorname{sd}_{p_0}(\|X-x_{\mathrm c}\|)
\right|,
\label{eq:empirical-radial-diagnostic}
\end{equation}
where $\widehat{\mathbb E}$ and $\widehat{\operatorname{sd}}$ denote the empirical mean and standard deviation computed from the generated samples, respectively.
For the centered targets considered below, we choose \(x_{\mathrm c}=0\). This
supplementary quantity detects coarse discrepancies in radial location and
spread, but it does not compare the complete radial distributions and is not
used as a stand-alone measure of geometric agreement.

The diagnostics in
\eqref{eq:empirical-covariance-error}--\eqref{eq:empirical-radial-diagnostic}
capture complementary aspects of the generated law. Nevertheless, covariance
and coordinate-wise marginal agreement do not determine a general
non-Gaussian joint distribution. They are therefore interpreted together
with target-specific summaries and, in the two-dimensional experiments,
visual comparisons of the generated and reference distributions.

These empirical diagnostics should not be identified with the error terms in
Theorem~\ref{thm:tv-z-chen}. In particular, they do not separately estimate
the finite-time initialization error, the Euler discretization error
\(\varepsilon_{\mathrm{disc}}(T)\), or the learned-control error
\(\varepsilon_Z(T)\), and they do not provide a bound in total variation.
Their role is to summarize observable discrepancies between the generated
particles and the reference target.

The numerical examples use target-specific parametrizations and stabilization
choices. The Gaussian-mixture examples use the direct \(Z\) parametrization,
whereas the rings example uses the hybrid parametrization together with the
radius feature and the consistency penalty. Accordingly, the configurations
reported in Table~\ref{tab:selected_rb_methods} are representative
configurations obtained during the numerical development of each example,
rather than the output of a universal model-selection rule. We report the
underlying diagnostics directly and do not combine them into a
target-dependent scalar score.

The effects of the principal numerical choices---on-policy replay,
\(Z\)-process clipping, the rollout loss, the time-grid resolution, and the
time horizon---are examined separately in the ablation study.

\begin{table}[t]
\centering
\setlength{\tabcolsep}{7pt}
\renewcommand{\arraystretch}{1.08}
\resizebox{\textwidth}{!}{%
\begin{tabular}{llrrrr}
\toprule
Target & Selected RBDS run & $G_{\mathrm{rad}}$ & $E_{\mathrm{cov}}$ & $W_{1,\mathrm{marg}}$ & Loss \\
\midrule
3-mode mixture & $T=3.2$, direct $Z$, no rollout, $K=2000$ & 0.316 & 0.82 & 0.423 & 0.00295 \\
Spaced mixture & $\beta=2$, direct $Z$, $K=2000$ & 1.54 & 10.5 & 1.97 & 0.0585 \\
Light-tail & $T=2.6$, gradient $Z$, adaptive clipping & 0.0497 & 0.0447 & 0.0518 & 5.28 \\
High-covariance & $T=2.6$, gradient $Z$, adaptive clipping & 0.241 & 0.52 & 0.187 & 0.433 \\
Banana & $T=2.6$, gradient $Z$, $K=2000$ & 0.123 & 0.233 & 0.122 & 0.557 \\
Rings & $T=3$, $\beta=1.1$, hybrid $Z$, radius feature & 1.33 & 8.54 & 1.44 & 1.72 \\
\bottomrule
\end{tabular}

}

\caption{Representative RBDS configurations and corresponding diagnostic
values for each target. The configurations are target-specific, reflecting
the distinct geometric features emphasized by the examples. Lower values
indicate closer agreement with the reference target.}
\label{tab:selected_rb_methods}
\end{table}

\subsection{Comparison with MALA}

Table~\ref{tab:selected_vs_mala} compares the selected RBDS variants with MALA. MALA is very strong on unimodal or mildly multimodal low-dimensional targets, especially the moderately separated mixture and the high-correlation Gaussian. The RBDS sampler is more competitive on targets where global transport structure or boundary geometry matters: the light-tailed quartic target, the spaced mixture, the banana target, and the rings target. For the spaced mixture and rings target, MALA can remain well centered while having poor radial geometry, covariance, and marginal Wasserstein errors, which indicates poor global exploration of separated components or radii.

\begin{table}[t]
\centering
\footnotesize
\setlength{\tabcolsep}{4pt}
\renewcommand{\arraystretch}{1.05}

\begin{tabular}{lrrrl}
\toprule
Target & MALA geom. & MALA cov. & MALA $W_1$ & Score-favored \\
\midrule
3-mode mixture & 0.118 & 0.448 & 0.139 & MALA \\
Spaced mixture & 6.7 & 33.3 & 5.93 & RBDS \\
Light-tail & 0.183 & 0.13 & 0.0493 & RBDS \\
High-covariance & 0.0437 & 0.0848 & 0.0321 & MALA \\
Banana & 0.397 & 0.879 & 0.163 & RBDS \\
Rings & 5.43 & 22.5 & 4.56 & RBDS \\
\bottomrule
\end{tabular}

\caption{MALA diagnostic values for each target, together with the score-based comparison against the representative RBDS configurations reported in Table~\ref{tab:selected_rb_methods}. The final column provides a descriptive summary based on the target-specific diagnostic score and should not be interpreted as a formal benchmark ranking. Although MALA is a strong local-gradient baseline in low dimensions, it may mix poorly across widely separated modes or annular structures.}
\label{tab:selected_vs_mala}
\end{table}



\subsection{Ablation study}

Table~\ref{tab:ablation_scores} reports a compact ablation score for six 500-step RBDS variants: the default configuration, no on-policy replay, no $Z$ process clipping, no rollout loss, a coarser time grid, and a longer horizon. Scores are computed within each target, so they should be compared row-wise only. The results show that no single ablation dominates all targets. Removing the rollout term is useful for the spaced mixture and rings examples, but it degrades the banana and high-correlation examples. Removing clipping improves the high-correlation and banana short ablations, but it is unstable for the light-tailed target and produces non-finite samples. This supports using adaptive clipping as the safer default. 

\begin{table}[H]
\centering
\setlength{\tabcolsep}{5pt}
\renewcommand{\arraystretch}{1.08}
\resizebox{\textwidth}{!}{%
\begin{tabular}{llrrrrrrl}
\toprule
Target & Diagnostic & Default & No replay & No clipping & No rollout & Coarse grid & Longer $T$ & Best ablation \\
\midrule

\multirow{3}{*}{3-mode mixture} & $G_{\mathrm{rad}}$ & 0.543 & 0.464 & 0.577 & 0.543 & 0.688 & 0.543 & \multirow{3}{*}{No replay} \\
 & $E_{\mathrm{cov}}$ & 2.62 & 2.11 & 2.68 & 2.62 & 3.66 & 2.62 &  \\
 & $W_{1,\mathrm{marg}}$ & 0.707 & 0.583 & 0.747 & 0.707 & 0.981 & 0.707 &  \\
\addlinespace

\multirow{3}{*}{Spaced mixture} & $G_{\mathrm{rad}}$ & 5.2 & 4.5 & 5.68 & 3.45 & 5.65 & 4.2 & \multirow{3}{*}{No rollout} \\
 & $E_{\mathrm{cov}}$ & 30.5 & 28.3 & 31.6 & 24.1 & 31.7 & 27.3 &  \\
 & $W_{1,\mathrm{marg}}$ & 5.26 & 4.9 & 5.47 & 4.17 & 5.32 & 4.44 &  \\
\addlinespace

\multirow{2}{*}{Light-tail} & $E_{\mathrm{cov}}$ & 0.193 & 0.0652 & ---\textsuperscript{$\dagger$} & 0.0712 & 0.338 & 0.157 & \multirow{2}{*}{No rollout} \\
 & $W_{1,\mathrm{marg}}$ & 0.837 & 0.887 & ---\textsuperscript{$\dagger$} & 0.846 & 0.271 & 0.831 &  \\
\addlinespace

\multirow{2}{*}{High-covariance} & $E_{\mathrm{cov}}$ & 1.08 & 295 & 0.105 & 1.2 & 0.404 & 0.9 & \multirow{2}{*}{No clipping} \\
 & $W_{1,\mathrm{marg}}$ & 0.557 & 7.04 & 0.786 & 1.42 & 0.348 & 0.496 &  \\
\addlinespace

\multirow{2}{*}{Banana} & $E_{\mathrm{cov}}$ & 12.1 & 37.4 & 4.4 & 13.5 & 6.95 & 6.87 & \multirow{2}{*}{No clipping} \\
 & $W_{1,\mathrm{marg}}$ & 3.68 & 6.9 & 1.63 & 5.16 & 3.14 & 3.69 &  \\
\addlinespace

\multirow{3}{*}{Rings} & $G_{\mathrm{rad}}$ & 2.79 & 2.59 & 2.61 & 0.218 & 2.91 & 2.66 & \multirow{3}{*}{No rollout} \\
 & $E_{\mathrm{cov}}$ & 15.5 & 14.7 & 14.7 & 2.17 & 16.1 & 15 &  \\
 & $W_{1,\mathrm{marg}}$ & 2.82 & 2.61 & 2.62 & 1.25 & 2.94 & 2.68 &  \\
\bottomrule
\end{tabular}

}
\caption{Raw diagnostic errors for the 500-step RBDS ablations.
Lower values are better.
\textsuperscript{\(\dagger\)}The run produced non-finite samples,
so the diagnostics could not be computed.}
\label{tab:ablation_scores}
\end{table}


\subsection{Selected sample plots}

Figures~\ref{fig:LightTail2D}--\ref{fig:Rings2D} show representative samples
from the selected RBDS variants together with target contours,
one-dimensional marginals, and training-loss traces.

\begin{figure}
\centering

\safeincludegraphics[width=\linewidth]{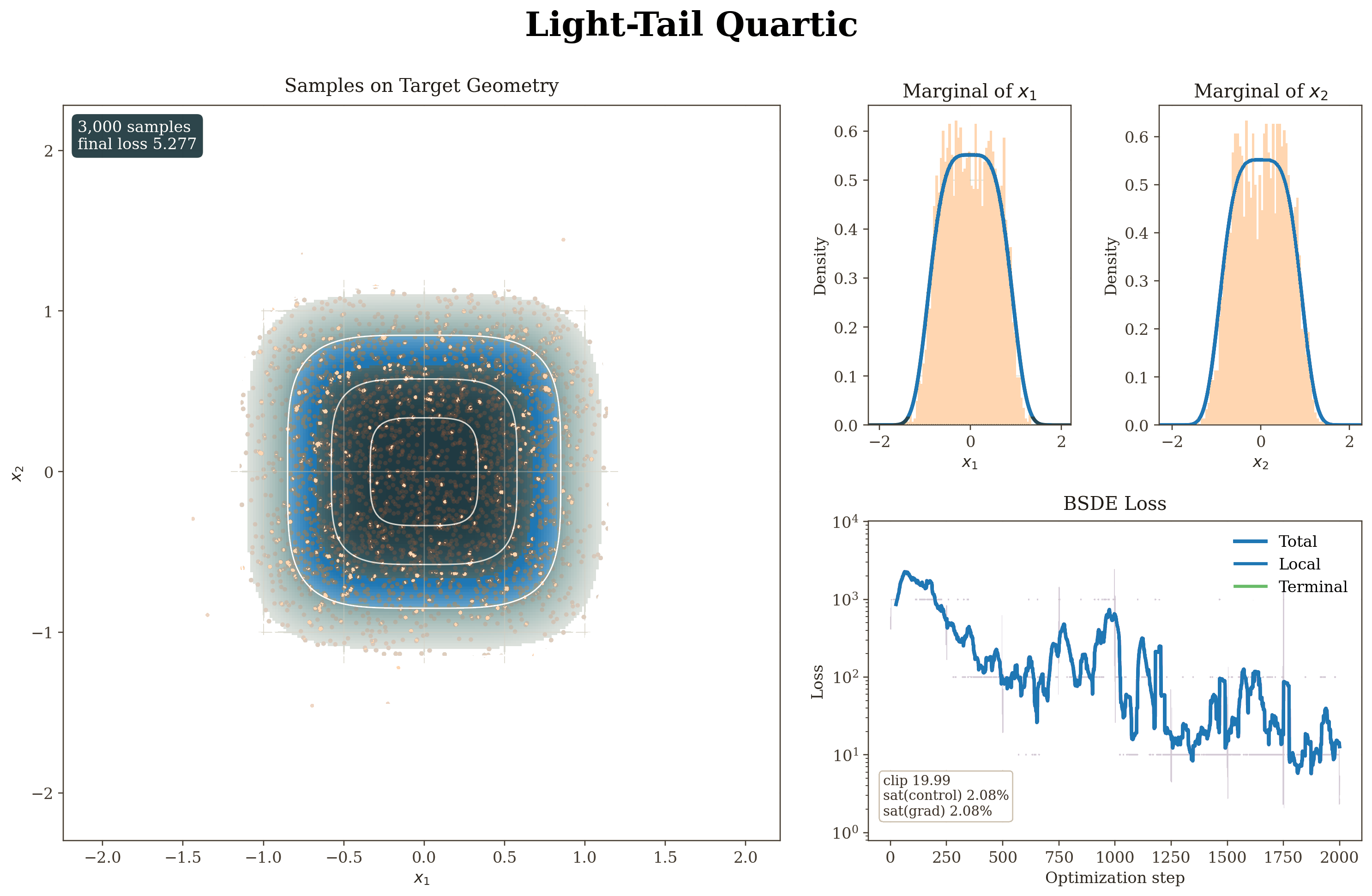}
    \caption{Quartic light-tailed target with non-Gaussian geometry.}
    \label{fig:LightTail2D}
\end{figure}

\begin{figure}
    \centering
    \safeincludegraphics[width=\linewidth]{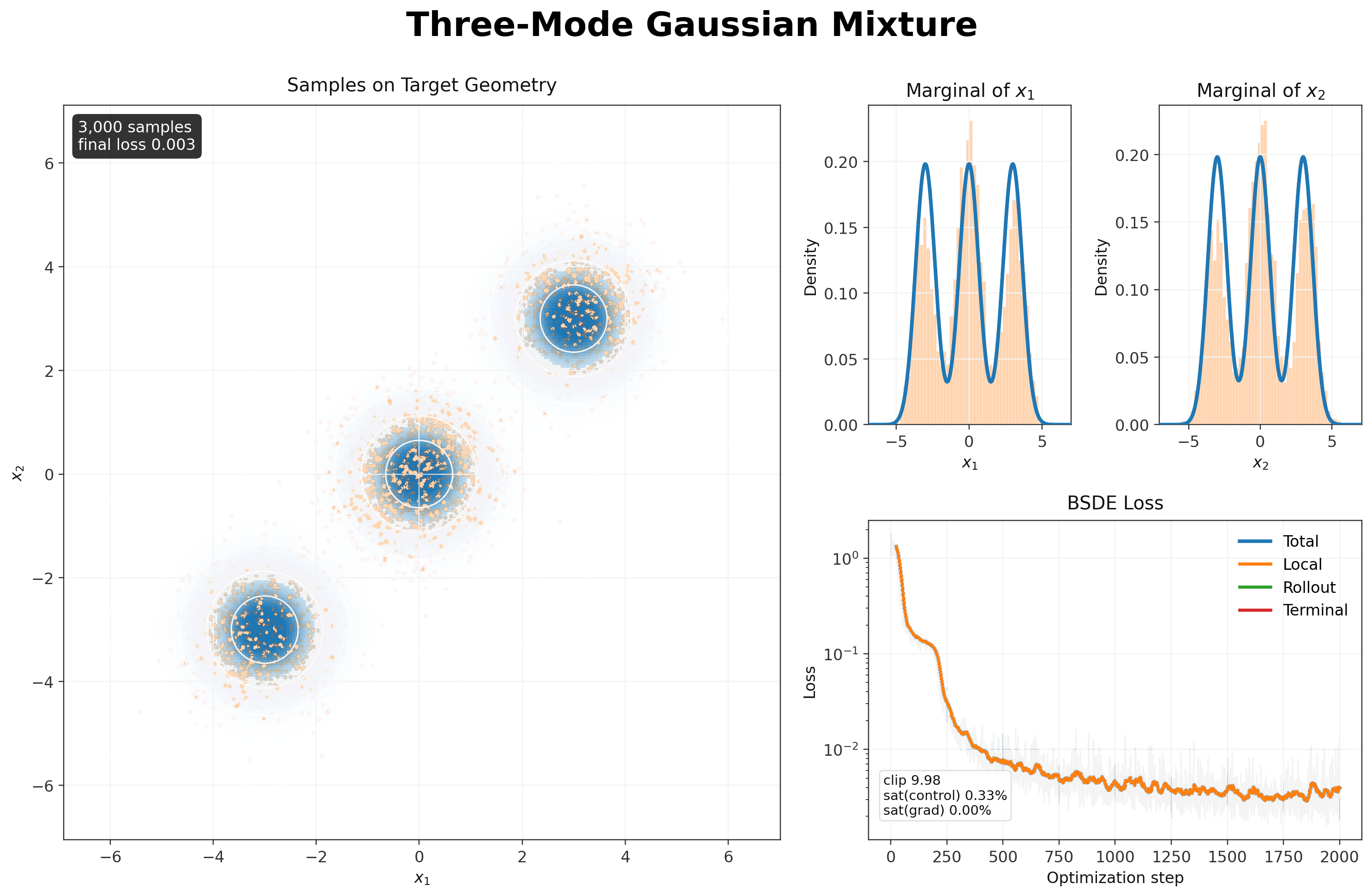}
    \caption{Three-component Gaussian mixture with moderately separated modes.}
    \label{fig:MixtureGaussian2D}
\end{figure}

\begin{figure}
    \centering

    \safeincludegraphics[width=\linewidth]{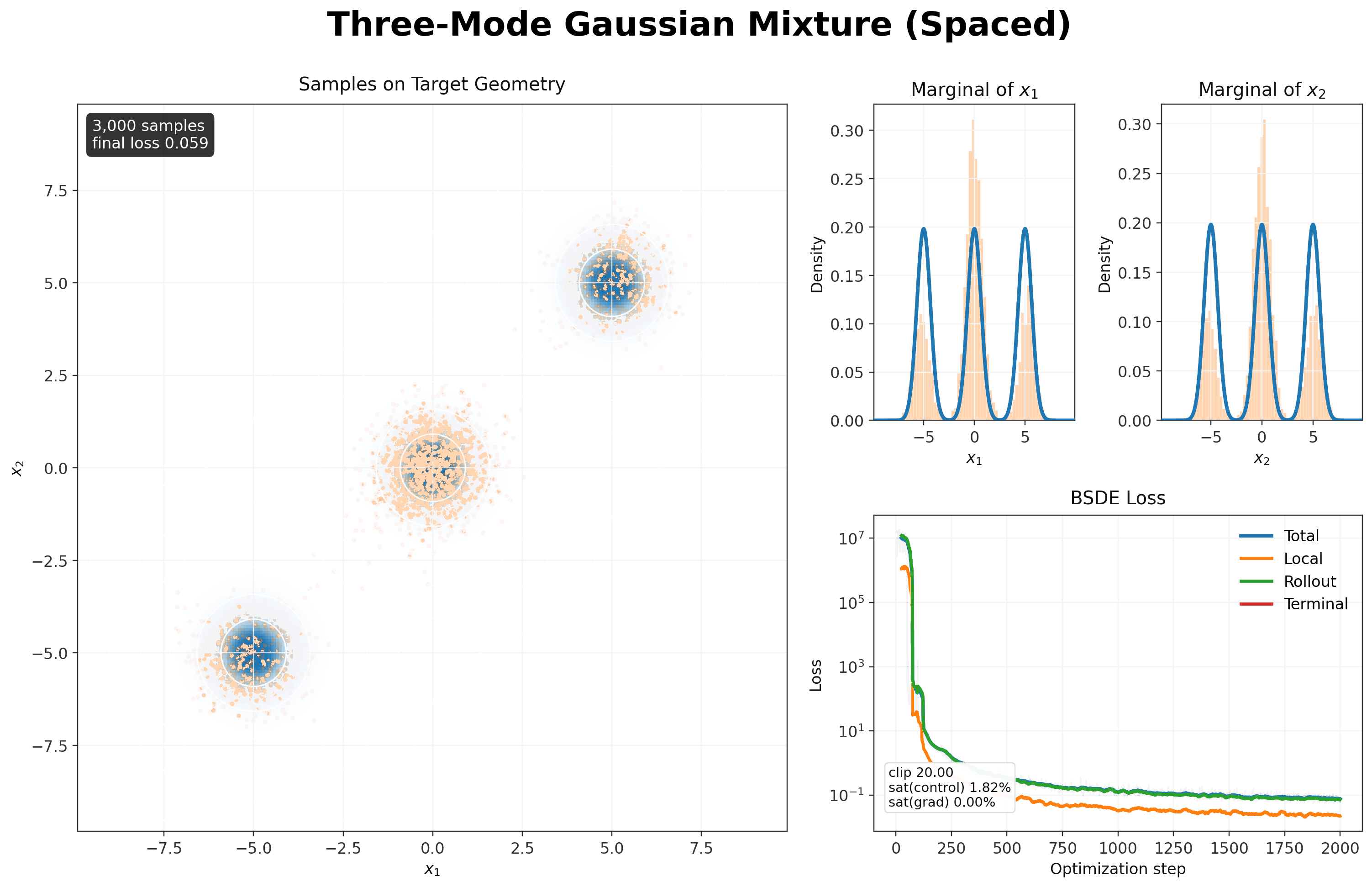}

    \caption{Three-component Gaussian mixture with well-separated modes.}
    \label{fig:MixtureGaussianSpaced2D}
\end{figure}

\begin{figure}
    \centering
    \safeincludegraphics[width=\linewidth]{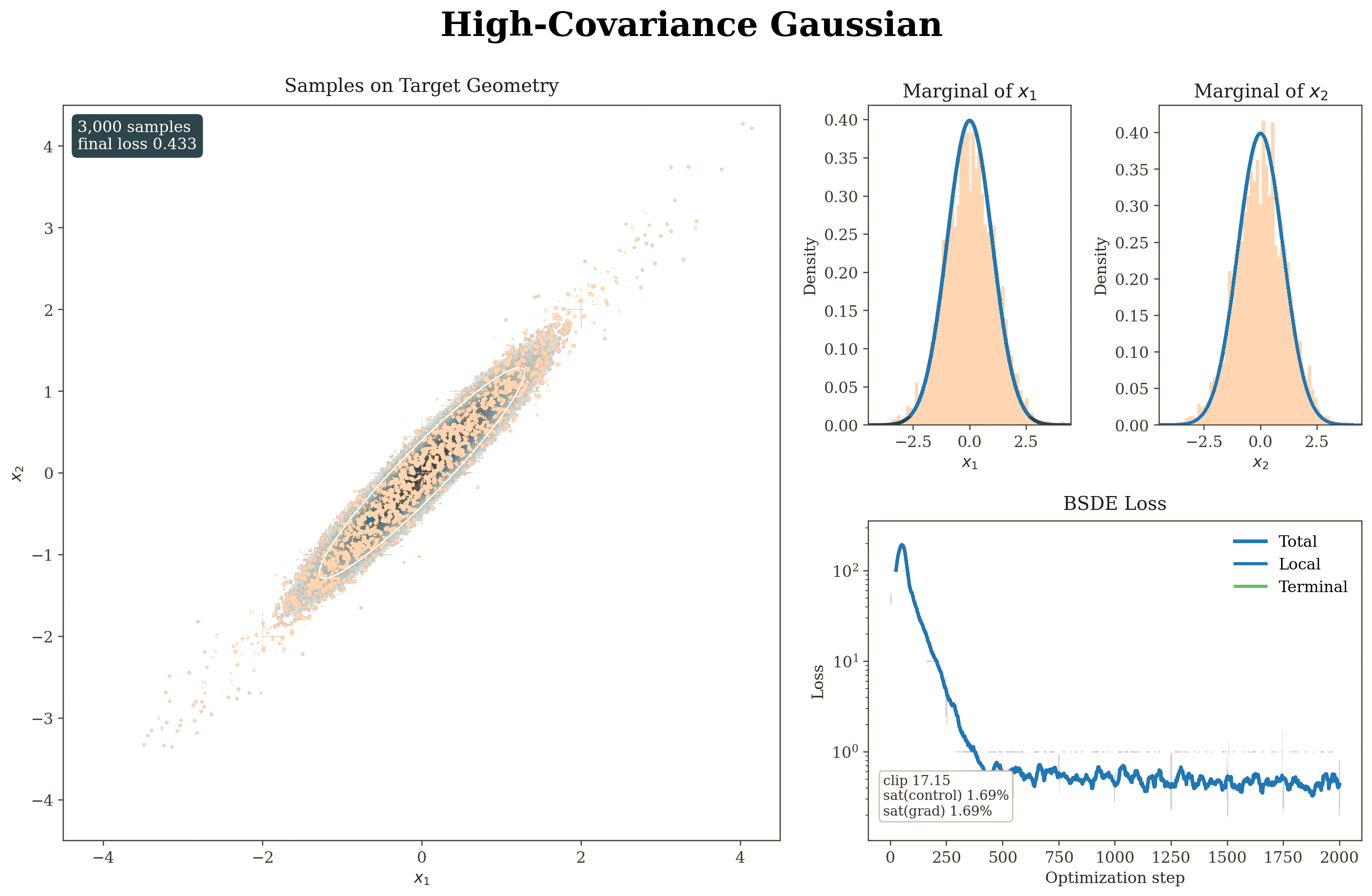}

    \caption{Gaussian target with strong correlation and large variance.}
    \label{fig:HighCovariance2D}
\end{figure}

\begin{figure}
    \centering
        \safeincludegraphics[width=\linewidth]{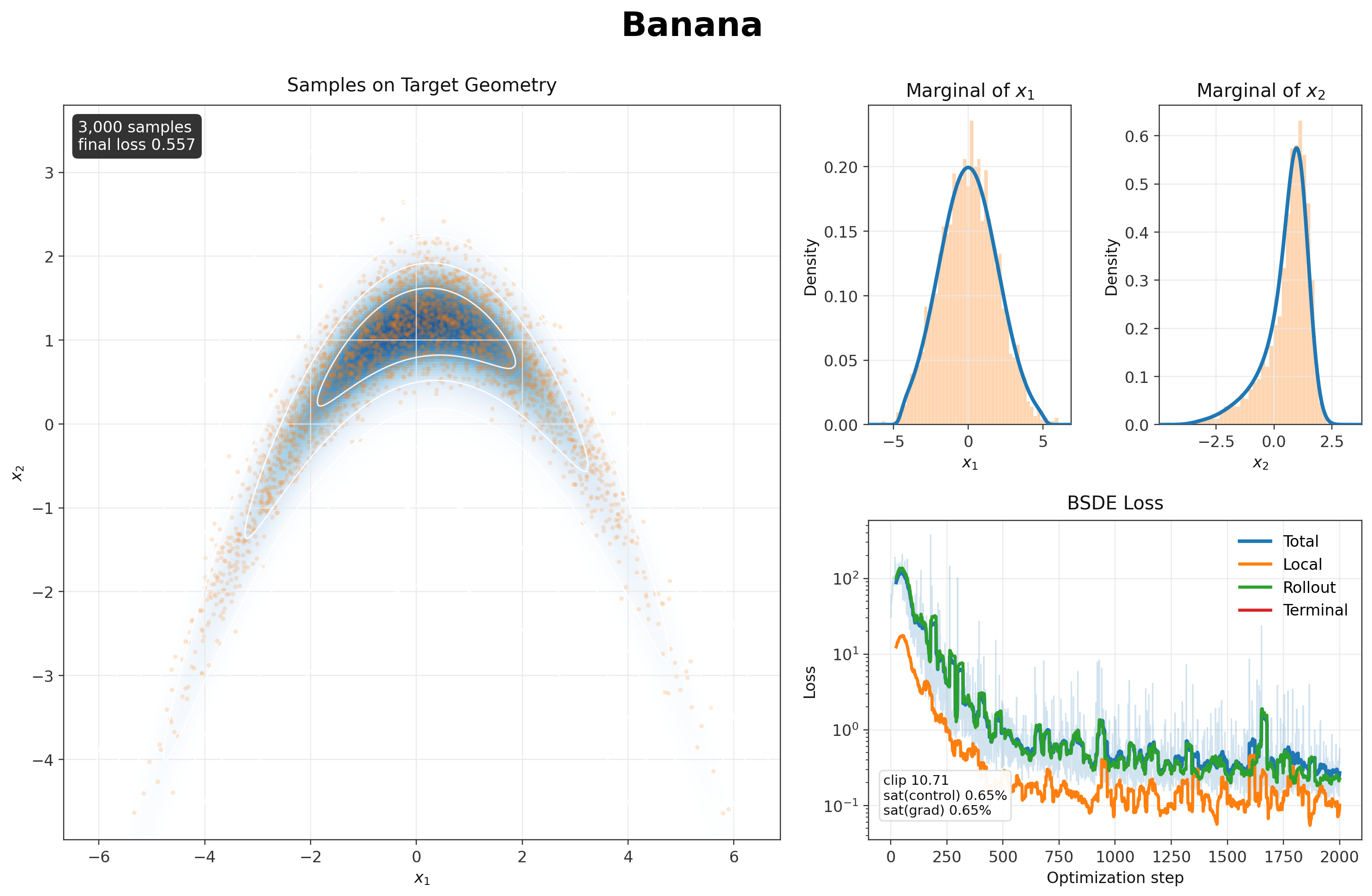}

    \caption{Banana-shaped target distribution with pronounced curvature.}
    \label{fig:Banana2D}
\end{figure}

\begin{figure}
    \centering
        \safeincludegraphics[width=\linewidth]{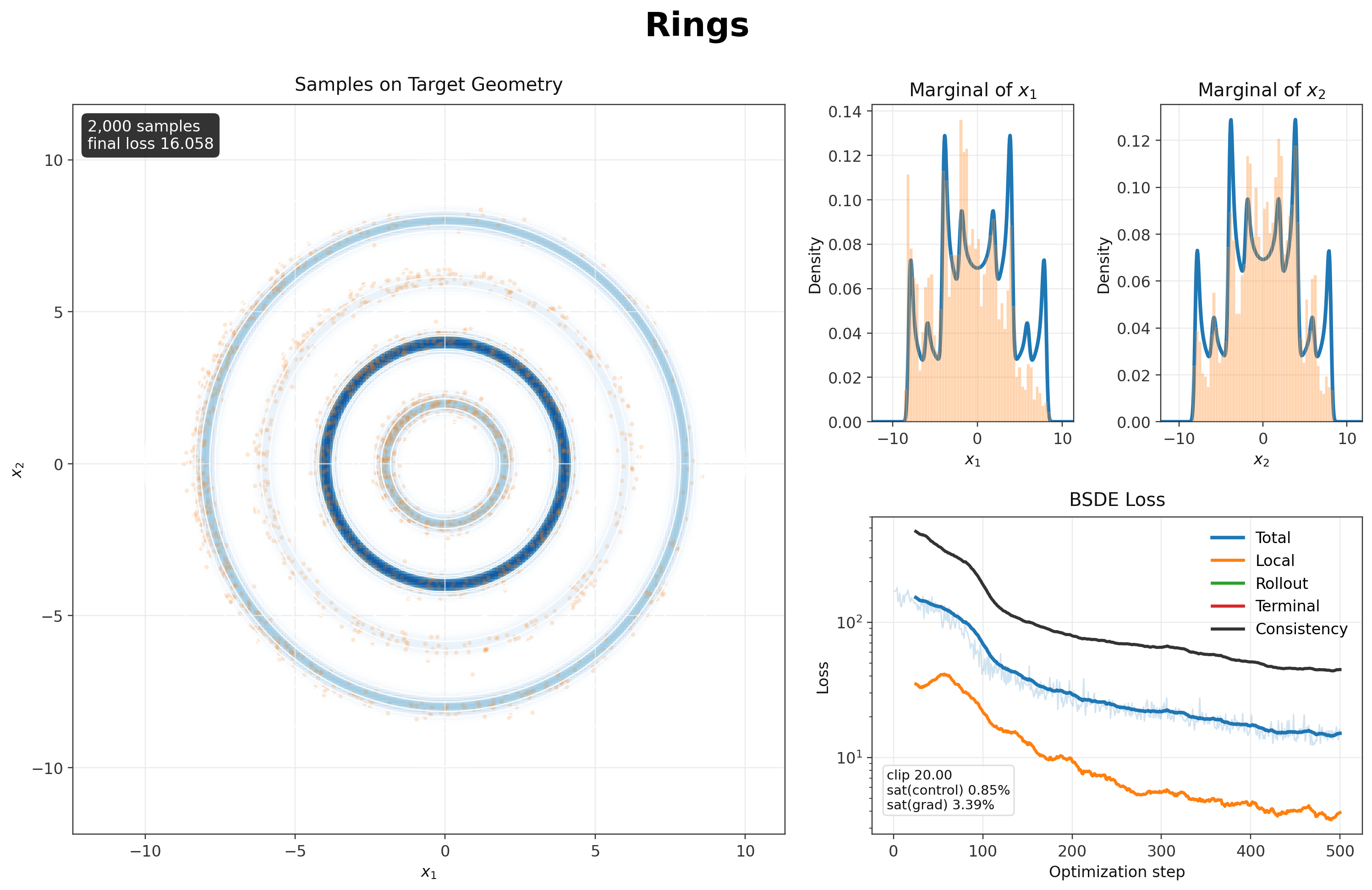}
    \caption{Concentric-rings target distribution. This is the only example in which we use the hybrid $Z$ parametrization together with the consistency regularization term.}
    \label{fig:Rings2D}
\end{figure}

\section{Conclusion}

We introduced a Reverse-BSDE Diffusion Sampler (RBDS) for sampling from targets specified through unnormalized densities. The method formulates diffusion-based sampling through a reverse-time BSDE representation and trains the sampler using local residual losses along simulated trajectories. This provides a principled way to use the FBSDE framework for sampling unnormalized densities.

The numerical experiments show that the method is a promising proof of concept, especially for targets with nontrivial global structure. Several directions remain open. These include improving the stability of the FBSDE training objective, developing sharper error estimates for the learned sampler, and extending the method to higher-dimensional targets and more challenging posterior sampling problems.

\newpage

\appendix

\section{BSDE primer}\label{app:bsde}

A Backward Stochastic Differential Equation (BSDE) is a stochastic differential equation where the final value is known (albeit random). Its study started during the 1970s in the stochastic optimal control literature. There are several textbooks on the subject, as for instance, \cite{BSDEbook}. 

Naively, one could think that integration of the usual SDE from $t$ to $T$ would give a first guess of what a BSDE would be:
\[X_t = X_T - \int_t^T \mu(s,X_s)ds - \int_t^T \sigma(s,X_s)dW_s.\]
The problem is that the right-hand side is $\cF_T$-measurable and the left-hand side must be $\cF_t$-measurable, where $(\cF_t)_{t \in [0,T]}$ is the Brownian filtration. Hence, we cannot choose $\sigma$ arbitrarily. In fact, it is necessary to allow the volatility process to be part of the solution. More precisely, we must consider
\begin{equation}\label{eq:bsde-app}
Y_t = \xi + \int_t^T f(s,Y_s, Z_s) ds - \int_t^T Z_s dW_s,
\end{equation}
where $\xi \in L^2(\cF_T)$ is the final condition for $Y$: $Y_T = \xi$. The function $f$ is called the generator of the BSDE. A solution to this BSDE is an $\cF$-adapted pair $(Y_t,Z_t)$ that satisfies the equation above. 

It is important to notice two things: being adapted to $(\cF_t)_{t \in [0,T]}$ is the most important requirement. If we dropped it, the BSDE would have trivial solutions like $Y_t = \xi$ and $Z_t = 0$ (in the case of $f(s,0,0) \equiv 0$). Moreover, if $\xi \in \cF_0$, this previous argument would hold and the trivial solution would be valid. Therefore, it is paramount to have $\xi$ measurable with respect to $\cF_T$ and the pair $(Y,Z)$ adapted to $(\cF_t)_{t \in [0,T]}$.

Notice as well that the BSDE~\eqref{eq:bsde-app} can also be written as an SDE (with unknown initial value):
\begin{align}\label{eq:bsde_sde}
Y_t = Y_0 - \int_0^t f(s,Y_s, Z_s)ds + \int_0^t Z_s dW_s,
\end{align}
or
\[dY_t = -f(t,Y_t, Z_t)dt + Z_tdW_t.\]
A \textit{Forward}-BSDE (FBSDE) adds an SDE to specify the dynamics of the final condition $\xi$ and the driver $f$:
\begin{align*}
&X_t = x + \int_0^t \mu(s,X_s)ds + \int_0^t \sigma(s,X_s)dW_s,\\
&Y_t = g(X_T) + \int_t^T f(s,X_s,Y_s, Z_s) ds - \int_t^T Z_s dW_s.
\end{align*}
Additionally, one could also consider that $(Y,Z)$ impacts the drift and volatility of the forward process. This would be called a fully-coupled FBSDE.

\section{Reverse SDE}\label{app:reverse}

The existence of the reverse SDE, first proved in \cite{ANDERSON1982}, was also studied in \cite{Haussmann1986} with simpler assumptions that are easier to verify. We state them here under the setting:
\begin{align*}
\begin{cases}
    d\cX_t = b(t, \cX_t)dt + \sigma(t) dW_t, \\
    \cX_0  \sim p_0,
\end{cases}
\end{align*}
where $b : [0,T] \times \mathbb{R}^n \longrightarrow \mathbb{R}^n$ and $\sigma : [0,T] \longrightarrow \mathbb{R}_{>0}$. In this case, the reverse SDE holds under the following assumptions on $b$ and $\sigma$:
\begin{enumerate}[label=(\alph*)]
\item There exists $K > 0$ such that for all $t\in [0, T]$ and $\bx, \by \in \R^n$
\begin{align*}
\begin{cases}
&\|b(t,\bx) - b(t,\by)\| \leq K \|\bx - \by\|,\\    
&\|b(t,\bx)\| + |\sigma(t)| \leq K(1 + \|\bx\|).
\end{cases}
\end{align*}
\item There exists $\alpha > 0$ such that $\sigma(t)^2 \geq \alpha$, for all $t \in \R$.
\item Assumption~\ref{ass:integrability} holds.
\end{enumerate}
The main theorem in \cite{Haussmann1986} states that, under these assumptions, the reverse process $\barcX$ satisfies

\begin{align*}
d\barcX_t
={}&\left(-b(T-t,\barcX_t)
+\sigma^2(T-t)\nabla\log p(T-t,\barcX_t)\right)dt+\sigma(T-t)d\overline W_t .
\end{align*}

The proof of the next lemma is a verification of the result described above: 
\begin{lemma}\label{lemma:reverse}
    Suppose Assumption~\ref{ass:integrability} holds, and let $\beta:[0,T] \to \R_{>0}$ be any continuous function. If $\mathcal{X}_t$ is defined as in Equation~\eqref{eq:forward}, then there exists a Brownian motion $\overline{W}_t$ such that Equation~\eqref{eq:reverse} holds for $\barcX_t = \cX_{T-t}$.
\end{lemma}

\begin{proof}[Proof of Lemma \ref{lemma:reverse}.]

In our case, $b(t,\bx) = -\frac{1}{2} \beta(t)^2\bx$, and $\sigma(t) = \beta(t)$. Since $\beta$ is a continuous positive function on the compact interval $[0,T]$, there exist $A > 0$ and $\alpha>0$ such that $0 < \alpha \leq \beta(t) \leq A$. Hence, the condition $(a)$ holds since:
\begin{align*}
\|b(t,\bx)-b(t,\by)\|
=\frac{1}{2}\beta(t)^2\|\bx-\by\|
\leq \frac{A^2}{2}\|\bx-\by\| .
\end{align*}
\[ \|b(t,\bx)\| + |\sigma(t)| \leq \frac{1}{2}A^2 \|\bx\| + A, \]
and condition $(b)$ holds because $\sigma(t)^2 \geq \alpha^2$. Obviously, condition $(c)$ holds under our assumptions. This concludes the proof.
    
\end{proof}

\section{Proofs}\label{app:uniqueness}

\begin{proof}[Proof of Theorem~\ref{thm:bsde}.]
By Lemma~\ref{lemma:reverse}, the reverse SDE \eqref{eq:reverse} holds for the reversed process $\barcX$ since Assumption~\ref{ass:integrability} holds. Moreover, $u$ given by \eqref{eq:u} is smooth and well defined; hence, by Itô's formula, we find that $(\barcX_t,u(t,\barcX_t),  \barbeta(t)\nabla u(t,\barcX_t))$ is a solution of \eqref{eq:bsde}.

Uniqueness is established by a Girsanov argument. Suppose that $(X_t, Y_t, Z_t)$ is a solution of~\eqref{eq:bsde} in the class stated in the theorem. Define $Y_t^1 = u(t, X_t)$ and $Z_t^1 = \barbeta(t) \nabla u(t,X_t)$. First, we prove that $Y_t = Y_t^1$ and $Z_t = Z_t^1$. By Itô's formula and Equation~\eqref{eq:uPDE}, applied along the candidate forward process $X$, we obtain
\begin{align}\label{eq:bsde-nofoward-right}
\begin{cases}
-dY^1_t
= \left(\dfrac{1}{2}\|Z^1_t\|^2-Z^1_t\cdot Z_t\right)dt -Z^1_t d\overline W_t,\\
Y^1_T=\log\pi(X_T).
\end{cases}
\end{align}
On the other hand, the candidate solution satisfies
\[
-dY_t=-\frac12\|Z_t\|^2dt-Z_t\,d\overline W_t,
\qquad Y_T=\log\pi(X_T).
\]
Define $\widehat{Y}_t = Y_t^1 - Y_t$, $\widehat{Z}_t = Z_t^1 - Z_t$ and $\eta_t = \frac{1}{2}(Z_t^1 - Z_t)$. Subtracting the two equations gives
\begin{align}\label{eq:difference-BSDE}
\widehat Y_t
= \int_t^T \eta_s\cdot\widehat Z_s\,ds
- \int_t^T \widehat Z_s\,d\overline W_s .
\end{align}
Hence,
\[ \widehat{Y}_t = - \int_t^T \widehat{Z}_s d B_s, \]
where
\[B_t = \overline{W}_t - \int_0^t \eta_s ds.\]

By the first admissibility condition in \eqref{eq:admissibility}, $\eta$ satisfies Novikov's condition. Girsanov's theorem therefore gives an equivalent measure $\mathbb{Q}$ under which $B$ is a Brownian motion. Under $\mathbb Q$, Equation~\eqref{eq:difference-BSDE} becomes
\[
\widehat Y_t=-\int_t^T \widehat Z_s\,dB_s.
\]
Now, by the second admissibility condition in \eqref{eq:admissibility}, the stochastic integral is a true martingale under $\mathbb{Q}$. Since $\widehat Y_T=0$, we conclude $\widehat Y_t=\mathbb E^\mathbb Q[\widehat Y_T\mid\mathcal F_t]=0$. Its quadratic variation then implies $\widehat Z=0$ $dt\otimes d\mathbb Q$-a.e. Since $\mathbb{Q}$ is equivalent to the original measure, $Y_t = Y_t^1$ and $Z_t = Z_t^1$ a.s. in the original measure. 

Therefore, $Y_t = u(t, X_t)$ and $Z_t = \barbeta(t)\nabla u(t, X_t)$. Hence, using Equation~\eqref{eq:u} and the dynamics of $X_t$, we find that
\begin{align*}
\begin{cases}
dX_t
= \Bigl(\tfrac{1}{2}\barbeta^2(t)X_t
+\barbeta^2(t)\nabla\log p(T-t,X_t)\Bigr)dt +\barbeta(t)d\overline W_t,\\
X_0=\barcX_0 .
\end{cases}
\end{align*}
	and, finally, by uniqueness of strong solutions for the reverse SDE~\eqref{eq:reverse}, which holds because of Assumption~\ref{ass:lipschitz-control}, $X_t = \barcX_t$ for all $t\in[0,T]$ almost surely.

\end{proof}

\begin{proof}[Proof of Theorem~\ref{thm:tv-z-chen}.]

Applying the bound of \cite{chen2023sampling} after the time change
$s(t)=\frac12\int_0^t\beta^2(r)\,dr$ gives the stated initialization
and discretization terms. The only remaining point is to translate the
score-estimation term into our FBSDE framework. In the present reverse-time notation,
\[
\nabla\log p(T-t,x)
=
\bar\beta(t)^{-1}z^\star(t,x),
\quad
\widehat s(t,x)
=
\bar\beta(t)^{-1}\widehat z(t,x).
\]
Therefore,
\begin{align*}
&\bar\beta^2(t)
\left\|
\widehat s(t,x)-\nabla\log p(T-t,x)
\right\|^2=
\left\|
\widehat z(t,x)-z^\star(t,x)
\right\|^2.    
\end{align*}
Thus the score-learning term in \cite{chen2023sampling} is exactly the
$Z$ process error $\varepsilon_Z$, while the initialization and discretization
terms are unchanged. This proves the claim.
\end{proof}

\bibliographystyle{plainnat}
\bibliography{biblio}

\end{document}